\documentclass[11pt]{article}

\usepackage{microtype}
\usepackage{graphicx}
\usepackage{subcaption}
\usepackage{booktabs}
\usepackage{wrapfig}
\usepackage{hyperref}

\usepackage[preprint]{tmlr}

\usepackage{amsmath}
\usepackage{amssymb}
\usepackage{mathtools}
\usepackage{amsthm}

\usepackage[capitalize,noabbrev]{cleveref}

\theoremstyle{plain}
\newtheorem{theorem}{Theorem}[section]
\newtheorem{proposition}[theorem]{Proposition}
\newtheorem{lemma}[theorem]{Lemma}

\theoremstyle{definition}
\newtheorem{definition}[theorem]{Definition}

\theoremstyle{remark}

\begin{document}

\title{Sobolev--Ricci Curvature}

\author{
\name Kyoichi Iwasaki \email iwasaki.kyoichi@ism.ac.jp \\
\addr The Graduate University for Advanced Studies, SOKENDAI
\AND
\name Tam Le \email tam@ism.ac.jp \\
\addr The Institute of Statistical Mathematics
\AND
\name Hideitsu Hino \email hino@ism.ac.jp \\
\addr The Institute of Statistical Mathematics
}

\date{}
\maketitle

\begin{abstract}
Ricci curvature is a fundamental concept in differential geometry for encoding local geometric structure, and its graph-based analogues have recently gained prominence as practical tools for reweighting, pruning, and reshaping network geometry.
We propose Sobolev--Ricci Curvature (SRC), a graph Ricci curvature canonically induced by Sobolev transport geometry, which admits efficient evaluation via a tree-metric Sobolev structure on neighborhood measures.
We establish two consistency behaviors that anchor SRC to classical transport curvature: (i) on trees endowed with the length measure, SRC recovers Ollivier--Ricci curvature (ORC) in the canonical $W_1$ setting, and (ii) SRC vanishes in the Dirac limit, matching the flat case of measure-theoretic Ricci curvature.
We demonstrate SRC as a reusable curvature primitive in two representative pipelines. We define Sobolev--Ricci Flow by replacing ORC with SRC in a Ricci-flow-style reweighting rule, and we use SRC for curvature-guided edge pruning aimed at preserving manifold structure. 
Overall, SRC provides a transport-based foundation for scalable curvature-driven graph transformation and manifold-oriented pruning.
\end{abstract}

\section{Introduction}
Curvature provides a concise geometric primitive for graph-structured data: it summarizes how local neighborhoods contract or expand, and it can be injected into downstream procedures such as graph reweighting, denoising, and manifold-oriented edge pruning. Among discrete notions, Ollivier--Ricci curvature (ORC) has become particularly influential because it defines curvature through optimal-transport (OT) geometry between neighborhood measures, closely mirroring contraction-based characterizations in the smooth theory~\citep{ollivier2009ricci}.
Despite its geometric appeal, ORC is often impractical at scale. 
ORC involves solving an OT problem for each edge, which becomes the dominant cost in iterative pipelines. We remove this bottleneck: on a tree, Sobolev transport (ST) admits a closed-form evaluation (via cut-mass coordinates), so curvature evaluation reduces to simple algebraic operations rather than iterative optimization.
Evaluating ORC typically requires solving a $1$-Wasserstein OT problem for each edge, which becomes prohibitively expensive on large graphs or within iterative pipelines where curvature must be recomputed repeatedly. Approximate solvers can reduce the cost but still remain a major bottleneck in practice.
Recent progress on scalable transport metrics suggests a different route: ST equips probability measures on a graph with a Sobolev-type metric--measure geometry and admits efficient evaluation through a tree-metric structure induced by rooted graph partitions~\citep{DBLP:conf/aistats/LeNPN22}. In particular, ST yields a closed-form expression on tree supports while retaining a principled connection to Wasserstein geometry.
This closed-form property eliminates per-edge OT solves.
Motivated by this transport geometry, we introduce \emph{Sobolev--Ricci Curvature (SRC)}, a notion of Ricci curvature on graphs naturally induced by ST. 

SRC follows the same transport--contraction template as ORC---curvature is defined by a transport distance between neighborhood measures normalized by a base distance---with ST providing the underlying geometry and enabling efficient evaluation.
Beyond scalability, SRC is designed to inherit two canonical consistency behaviors expected of transport-based curvature: \emph{(i)} on trees endowed with the length measure, SRC recovers ORC in the classical $W_1$ transport formulation; and \emph{(ii)} in the Dirac limit where neighborhood measures collapse to point masses, SRC vanishes, aligning with the flat (zero-curvature) case in measure-theoretic formulations.

We further illustrate SRC as a \emph{geometrically consistent curvature primitive} in two pipelines. First, we define a \emph{Sobolev--Ricci Flow (SRF)} that updates edge lengths via an SRC-driven rule analogous to ORC-based flows, avoiding per-edge OT solves. Second, we propose \emph{SRC--MANL}, a Sobolev-curvature variant of ORC--MANL (Manifold Learning and Recovery)~\citep{DBLP:conf/iclr/SaidiHB25}, where SRC replaces ORC in the curvature-based filtering stage. Together, these examples show that SRC provides a principled ST curvature foundation for scalable curvature-driven graph transformation and manifold-oriented edge pruning.

\paragraph{Contributions.}
Our main contributions are:
\begin{itemize}
  \item \textbf{Sobolev--Ricci Curvature (SRC):} 
  a Sobolev transport–induced graph Ricci curvature, efficient tree-metric evaluation.
  \item \textbf{Consistency:} 
  equals ORC on length-measure trees; vanishes in the Dirac limit.
  \item \textbf{Two representative instantiations:} SRF for curvature-driven reweighting and SRC--MANL for manifold-oriented edge pruning, showcasing SRC as a geometrically consistent curvature primitive.
\end{itemize}

\paragraph{Organization of the paper.}
Sec.~\ref{sec:preliminaries} introduces preliminaries.
Sec.~\ref{sec:related_work} reviews ORC and related variants.
Sec.~\ref{sec:sobolev_curvature} presents SRC.
Sec.~\ref{sec:properties} analyzes its theoretical properties and connection to ORC.
Sec.~\ref{sec:ricci_flow} presents SRF and its complexity.
Sec.~\ref{sec:numerical_experiments} reports experiments on community detection and edge pruning.
Sec.~\ref{sec:conclusion} concludes with future directions.

\section{Preliminaries}\label{sec:preliminaries}

In this section, we introduce problem settings, notations, and related definitions.

\paragraph{Ricci curvature.} \
In Riemannian geometry, the \emph{Ricci curvature} quantifies the amount by which the volume of a geodesic ball in a manifold deviates from that in the Euclidean space~\citep{ricci1900methodes}. Intuitively, it measures the contraction or expansion of geodesics, and plays a central role in geometric analysis.

\paragraph{Ollivier--Ricci curvature.} \
\citet{ollivier2009ricci} introduced a discrete analogue, the \emph{Ollivier--Ricci curvature (ORC)}, which adapts Ricci curvature to metric measure spaces and graphs.
Given a metric space $(X,d)$ with probability measures $\{\mu_x\}_{x\in X}$ supported on local neighborhoods of each $x \in X$, the ORC between two nodes $x,y \in X$ is defined as
\begin{equation}
    \kappa_{\mathrm{ORC}}(x,y) 
    := 1 - \frac{W_1(\mu_x,\mu_y)}{d(x,y)},
    \label{eq:orc_def}
\end{equation}
where $W_1$ denotes the 1-Wasserstein distance. 
Positive curvature indicates that the neighborhood measures are closer (in the Wasserstein sense) than the base distance $d(x,y)$, while negative curvature indicates the opposite. 
This formulation provides a versatile geometric descriptor of graphs, but the computation of $W_1$ for each edge is expensive, motivating scalable transport-based formulations that admit efficient evaluation on graphs. See Appendix~\ref{sec:from_riemannian_rc_to_orc_on_graphs} for the relation between the original Ricci curvature and ORC.

\paragraph{Graph setup.} \
Let $X=\{x_i\}_{i=1}^n\subset\mathbb{R}^d$ and $G=(V,E)$ be a connected graph with $V=X$ and positive edge lengths $\{\lambda(e)\}_{e\in E}$ induced by the $\ell^p$ norm\footnote{For $u=(u_1,\dots,u_d)\in\mathbb{R}^d$, $\|u\|_p=(\sum_{k=1}^d |u_k|^p)^{1/p}$ for $1\le p<\infty$, and $\|u\|_\infty=\max_{1\le k\le d}|u_k|$.} on $X$.  
Fix a root node $z_0\in V$ such that $[z_0,v]$, which is the shortest-path from $z_0$ to $v$, is unique for every $v\in V$ (unique-path-to-root assumption)~\citep{DBLP:conf/aistats/LeNPN22}.\footnote{There may be multiple paths connecting $z_0$ and $v$, but the shortest path $[z_0, v]$ is unique.} 
In this work, we considered two types of trees, i.e., Minimum Spanning Tree (MST) and Shortest Path Tree (SPT); the assumption holds by construction.

\paragraph{Probability measures.} \
We denote by $\mathcal{P}(V)$ the simplex of probability measures on $V$.  
For $x\in V$, we may write $\mu_x\in\mathcal{P}(V)$ when we emphasize a probability measure associated with $x$.  
A canonical example is the Dirac measure $\delta_x \in \mathcal{P}(V)$, defined by $\delta_x(\{v\}) = 1$ if $v = x$, and $0$ otherwise.

\paragraph{Edge partitions and neighborhoods.} \
Fixing the root $z_0\in V$ and assuming the unique-path-to-root property, for each edge $e\in E$ we remove $e$ to obtain two connected components.  
Let $T_e\subseteq V$ be the component not containing $z_0$, equivalently
\begin{equation}\label{formula:t_e}
T_e=\{\, v\in V : e\in [z_0,v] \,\}.
\end{equation}
We denote by $\mathbf{1}_{T_e}$ the indicator of $T_e$, i.e. $\mathbf{1}_{T_e}(v)=1$ if $v\in T_e$ and $0$ otherwise.
For $x\in V$, we denote by $\mathcal{N}(x)\subseteq V$ the set of its $k$-nearest neighbors with respect to the $\ell^p$ norm.

\paragraph{Sobolev transport.} \
Sobolev transport (ST)~\citep{DBLP:conf/aistats/LeNPN22} equips graphs satisfying the unique-path-to-root assumption
with a tree-metric-based transport geometry, and admits a closed-form expression on such structures.

In this work, we instantiate ST through MST and SPT constructions.
For two probability measures $\mu_x,\mu_y \in \mathcal{P}(V)$, the ST admits a closed-form expression as follows: \begin{equation}
    S_p(\mu_x,\mu_y)^p 
    = \sum_{e \in E} \lambda(e)\, \big| \mu_x(\mathbf{1}_{T_e}) - \mu_y(\mathbf{1}_{T_e}) \big|^p,
    \label{eq:st_def}
\end{equation}
where $T_e$ is defined in~\eqref{formula:t_e}.

\section{Related work}\label{sec:related_work}
In this section, we review prior studies on ORC, its theoretical foundations, 
approaches for efficient computation, simplified curvature notions, and representative applications. 
We then highlight the recent development of ST, which motivates our proposed SRC.

\subsection{ORC and theoretical foundations}
ORC was introduced as a discrete analogue of Ricci curvature on Markov chains and metric spaces~\citep{ollivier2009ricci}.
It has been studied in relation to graph Laplacians~\citep{bauer2011ollivier}, the heat equation~\citep{munch2019ollivier}, and random matrix perspectives~\citep{DBLP:journals/corr/abs-2008-01209}. 
In parallel, diffusion-based approaches have been proposed for estimating curvature in data manifolds, offering a continuous viewpoint that complements graph-based formulations~\citep{bhaskar2022diffusion}.

These works collectively establish the theoretical soundness of ORC as a discrete geometric descriptor, though the exact computation involves solving OT problems with high computational cost.

\subsection{Efficient Computation of ORC and Variants}
Due to the high complexity of ORC computation, several acceleration methods have been proposed. 
Entropic regularization with Sinkhorn iterations provides a scalable approximation to the Wasserstein distance~\citep{DBLP:conf/nips/Cuturi13}. 
Recent studies further develop accelerated evaluation schemes for lower bounds of ORC~\citep{DBLP:journals/corr/abs-2405-13302} and algorithms with fine-grained reductions~\citep{DBLP:journals/tcs/DasGuptaGM23}. 

Beyond entropic regularization, sliced and tree-sliced variants of Wasserstein distances have been introduced to reduce computational cost while preserving geometric fidelity~\citep{DBLP:conf/nips/LeYFC19,DBLP:conf/nips/KolouriNSBR19}.

Other variants include Lin--Lu--Yau curvature~\citep{lin2011ricci, hehl2024ollivier}, which provides alternative discretizations of Ricci curvature, and more efficient evaluation strategies for ORC on large graphs~\citep{DBLP:conf/iclr/Fesser024}. 

\subsection{Alternative Discrete Curvature Notions}

Several non-transport discrete curvature notions have been proposed, including overlap-based curvatures (e.g., Jaccard) and combinatorial formulations (e.g., Forman curvature)~\citep{DBLP:journals/corr/abs-1710-01724,DBLP:journals/dcg/Forman03}.
Related directions also explore curvature-aware embeddings beyond purely discrete curvature definitions~\citep{DBLP:journals/corr/abs-2202-01185}.

\subsection{Applications of Discrete Curvature}
ORC and its variants have been applied in a wide range of domains.  
In network science, curvature-based methods have been used for community detection~\citep{DBLP:journals/corr/abs-1907-03993}, analyzing internet topology~\citep{DBLP:conf/infocom/NiLGGS15}, and wireless network robustness~\citep{DBLP:conf/amcc/WangJB14}.  
ORC has also been extended into manifold learning frameworks such as ORC-MANL~\citep{DBLP:conf/iclr/SaidiHB25}.  
In neuroscience, ORC and FC have been used to study brain structural networks~\citep{farooq2019network} and connectivity differences across populations~\citep{yadav2023discrete}.  
During the COVID-19 pandemic, curvature-based vulnerability analysis highlighted weak points in epidemic networks~\citep{de2020using}.  
In machine learning, curvature has been applied to address over-squashing in graph neural networks~\citep{DBLP:conf/iclr/ToppingGC0B22}, biomolecular interaction prediction~\citep{DBLP:journals/corr/abs-2306-13699}, analyzing diffusion processes on graphs~\citep{DBLP:journals/corr/abs-2106-05847}, and curvature-based clustering~\citep{tian2025curvature}.
Beyond these, curvature has also been exploited in computer vision tasks: manifold-regularized dynamic network pruning~\citep{DBLP:conf/cvpr/Tang0XD0T021} uses curvature-aware constraints for efficient model compression, and manifold-aligned neighbor embedding~\citep{DBLP:journals/corr/abs-2205-11257} leverages curvature information to enhance representation learning.

\subsection{Connection to Sobolev Transport}

Most recently, ~\citet{DBLP:conf/aistats/LeNPN22} proposed the Sobolev transport~(ST) as a scalable metric for probability measures on graphs, providing closed-form computation.
Relatedly, convolutional Wasserstein distances have been introduced for efficient optimal transportation on geometric domains~\citep{DBLP:journals/tog/SolomonGPCBNDG15}, highlighting how structural assumptions can drastically reduce computational costs in vision and shape analysis tasks.

Building upon this, our proposed SRC has been introduced as a curvature notion canonically induced by Sobolev transport geometry, with ORC recovered as a special degenerate case, offering comparable geometric fidelity at reduced computational cost. 
This perspective shifts the design axis from OT solvers to transport geometries: SRC treats ST as the underlying geometry and recovers ORC as a singular/degenerate case on trees.

Beyond discrete and graph-based settings, Sobolev-type metrics have also been investigated in the context of diffeomorphism groups and spaces of submanifolds~\citep{micheli2012sobolev}. 
While this line of work focuses on infinite-dimensional Riemannian geometry and applications 
such as shape analysis and fluid dynamics, it shares with our approach the underlying philosophy 
that Sobolev-type metrics provide a powerful tool to capture geometric structures beyond those 
accessible with $L^2$-based formulations.

\paragraph{Positioning within prior work.}
Transport-based curvatures define local geometry by comparing neighborhood measures.
ORC is a prominent instance, but per-edge OT can be prohibitive at scale,
while simplified alternatives such as JC and FC trade transport-based fidelity for efficiency.
ST provides a novel metric--measure geometry that admits closed-form evaluation on tree supports. 
We introduce \emph{Sobolev--Ricci Curvature (SRC)} as the Ricci curvature canonically induced by this Sobolev transport geometry.
In the special case of trees endowed with the length measure ($p=1$), SRC recovers ORC,
while more generally it enables scalable curvature evaluation without entropic or exact OT solvers.
Empirically, SRC instantiates an effective curvature primitive in SRC--MANL edge pruning, achieving ORC-level quality at substantially lower runtime.

\section{Sobolev--Ricci curvatures}\label{sec:sobolev_curvature}

We define neighborhood measures $\mu_x$ from either connectivity or feature-space neighborhoods, and compute ST on an induced rooted tree (MST or SPT), yielding SRC; the consistency results apply to any such instantiation (including the Dirac-limit regime discussed in Sec.~\ref{sec:properties}).

\paragraph{Neighborhood measures.}
We specify $\{\mu_x\}_{x\in V}\subset\mathcal{P}(V)$ supported on $\mathcal{N}(x)\subseteq V$ (e.g., one-hop neighborhoods or $k$NN under $\ell^p$).

\paragraph{Instantiations of $\mu_x$.}

We use a graph-based lazy random walk when features are unavailable, and a feature-based localization on $\mathcal{N}(x)$ (e.g., Gaussian weights) when they are available.

\textbf{(Graph-based / lazy random walk).}
We consider the standard lazy random walk neighborhood measure
\begin{equation}
\mu_x(v)=
\begin{cases}
\alpha, & v=x,\\
(1-\alpha)/|N(x)|, & v\in N(x),\\
0, & \text{otherwise},
\end{cases}
\label{eq:mu_flow}
\end{equation}
where $N(x)$ is the one-hop neighborhood of $x$.

\textbf{(Feature-based / Gaussian localization).}
When feature distances are available, we use a $k$NN-masked Gaussian localization
\begin{align}
\mu_x(v) &= \frac{1}{Z_x}\, \exp\!\left(-\frac{\|x-v\|_p^2}{\sigma^2}\right)\,\mathbf{1}_{\{v\in\mathcal{N}(x)\}}, \label{formula:mu_x_v} \\
Z_x &= \sum_{w\in\mathcal{N}(x)} \exp\!\left(-\frac{\|x-w\|_p^2}{\sigma^2}\right),\label{formula:z_x}
\end{align}

where we slightly abuse notation by identifying a node $x$ with its feature vector.

\paragraph{Graph-based tree induction}
Given an input graph $G=(V,E)$, we induce a rooted tree $T$ by selecting a subset of edges from $E$,
either as a minimum spanning tree (MST; global) or as a shortest-path tree (SPT) rooted at $z_0$ (local).
Here the tree edge length is inherited from the original graph (i.e., the edge weight/length, and set to $1$ if unavailable).
The MST provides a global backbone that connects all vertices with minimum total tree length,
whereas the SPT preserves shortest-path neighborhoods from the root along the original graph.

\paragraph{Feature-based tree induction}
When feature distances are available, we represent each node by its feature vector and measure
pairwise dissimilarity in the feature space; with a slight abuse of notation, we identify a node $x$
with its feature vector and write the feature-space distance as $\|x-v\|_p$.
We use this quantity as the tree edge length, and induce a tree that is short with respect to these
feature-space distances, for example, by computing an MST on the induced complete graph over $V$
with edge costs $\|u-v\|_p$ (optionally restricted to $k$NN edges for efficiency).

\paragraph{Tree-induced partition and ST.}
For each $e\in E$, let $T_e$ be the component not containing the root after removing $e$, and set $\gamma_e:=\mathbf{1}_{T_e}$.
We define $\mu_x(\gamma_e):=\sum_{v\in T_e}\mu_x(v)$.

The ST $S_p$ \citep{DBLP:conf/aistats/LeNPN22} between $\mu_x$ and $\mu_y$ is given in the discrete form by
\begin{equation}\label{def:s_p_via_st}
S_p(\mu_x,\mu_y)^p \;=\; \sum_{e\in E} \lambda(e)\, \bigl|\mu_x(\gamma_e)-\mu_y(\gamma_e)\bigr|^p.
\end{equation}
Eq.~\eqref{def:s_p_via_st} is Eq.~\eqref{eq:st_def} specialized to the induced rooted tree, written in terms of the cut-mass coordinates $\mu_x(\gamma_e)$.
For the Dirac measures, we denote $D_p(x,y)\;=\;S_p(\delta_x,\delta_y)$.

\paragraph{Sobolev--Ricci Curvature.}
The Sobolev--Ricci Curvature $\kappa_{\mathrm{S_p}}$ for a pair $(x,y)\in E$ is defined as
\begin{equation}\label{eq:sobolev_ricci_curvature}
\kappa_{\mathrm{S_p}}(x,y) \;=\; 1 \;-\; \frac{S_p(\mu_x,\mu_y)}{D_p(x,y)}.
\end{equation}
Note that on a tree, $D_p(x,y)=\big(\sum_{e\in[x,y]}\lambda(e)\big)^{1/p}$ since $|\delta_x(\gamma_e)-\delta_y(\gamma_e)|\in\{0,1\}$.

\section{Properties of Sobolev Curvature}\label{sec:properties}

In this section, we present fundamental properties of SRC; equivalence to ORC on trees, and consistency in the Dirac limit (measure-agnostic).

\subsection{ORC as a special case of SRC}

We recall that the ST $S_p$ depends on the choice of the
measure $\lambda$ on the underlying graph. An important case is when
$\lambda$ is the \emph{length measure} of the graph.

\begin{definition}[Length measure]
Let $G=(V,E)$ be a weighted graph, and let $d$ denote the graph metric induced
by the edge lengths $\{\lambda(e)\}_{e\in E}$. The \emph{length measure} $\lambda^\ast$
is defined as the unique Borel measure on $G$ such that for every edge
$e=\langle u,v\rangle$ with length $\lambda(e)$, and for any two points
$x=(1-s)u+sv$, $y=(1-t)u+tv$ with $0\le s \le t \le 1$, we have
\begin{equation}
  \lambda^\ast(\langle x,y\rangle) \;=\; |t-s|\, \lambda(e).    
\end{equation}

In particular, for any shortest path $[x,y] \subset G$, the measure satisfies
\begin{equation}
  \lambda^\ast([x,y]) = d(x,y).
\end{equation}

\end{definition}

With this choice of measure, ST with $p=1$ coincides with the $1$-Wasserstein distance on trees. This fact shows that ORC arises when the underlying Sobolev geometry degenerates to a tree equipped with the length measure.

\begin{proposition}
\label{prop:SRC_orc_equivalence}
Let $G$ be a tree endowed with the length measure $\lambda^*$. 
Then, for $p=1$, the SRC coincides with the ORC:
\begin{equation}
    \kappa_{\mathrm{S_1}}(x,y) = \kappa_{\mathrm{ORC}}(x,y).    
\end{equation}
\end{proposition}

The proof is provided in Appendix~\ref{subsec:SRC_orc_equivalence}.

\subsection{Consistency with Dirac Measures}\label{sec:consistency_with_dirac_measures}

A fundamental design principle of our SRC is its consistency with the behavior of ORC in the Dirac limit. 
In particular, for two Dirac measures $\delta_x$ and $\delta_y$ on $(X,d)$, the ORC $\kappa_{\mathrm{ORC}}$ is given by
\begin{equation}
\kappa_{\mathrm{ORC}}(\delta_x, \delta_y) 
= 1 - \frac{W_1(\delta_x, \delta_y)}{d(x,y)} 
= 1 - \frac{d(x,y)}{d(x,y)} = 0,
\end{equation}
since $W_1(\delta_x, \delta_y) = d(x,y)$.

Our proposed SRC naturally inherits the same limiting property. 
By construction, when the localized neighborhood distributions $\mu_x$ and $\mu_y$ converge to Dirac measures, the Sobolev functional reduces to $S_p(\delta_x, \delta_y)$
and the corresponding curvature is calculated as
\begin{equation}
\kappa_{\mathrm{S_p}}(\delta_x, \delta_y) = 1 - \frac{S_p(\delta_x, \delta_y)}{D_p(x,y)} = 0.    
\end{equation}

Thus, in the formal limit of measures $\mu_x\to \delta_x$ which is also considered as $\sigma\to 0$ in~\eqref{formula:mu_x_v}, both ORC and SRC vanish:
\begin{equation}
\lim_{\sigma \to 0} \kappa_{\mathrm{ORC}}(x,y) = \lim_{\sigma \to 0} \kappa_{\mathrm{S_p}}(x,y) = 0,    
\end{equation}
providing a natural justification that SRC can be viewed as a discrete analogue of Ricci curvature, aligned with the well-established behavior of ORC.
Dirac limit corresponds to the flat (zero-curvature) case in the measure-theoretic analogue of Ricci curvature.
This argument can be made rigorous via weak convergence of measures for both ORC and SRC in Appendix~\ref{sec:appendix-dirac-consistency}, and the same reasoning applies to the lazy random walk limit $\alpha\to 1$ in~\eqref{eq:mu_flow}.

\section{Sobolev--Ricci Flow}\label{sec:ricci_flow}

Ricci curvature is a widely used primitive to characterize local graph geometry and to drive curvature-based transformations such as discrete Ricci flow.
ORC-based Ricci flow iteratively reweights edges by curvature and has been shown to improve community structure and robustness~\citep{DBLP:journals/corr/abs-1907-03993},
but its practical use is limited by cost since ORC requires solving an OT problem per edge at every iteration.

We propose a curvature-driven flow based on SRC.
SRC retains the transport-based geometric interpretation while enabling efficient computation via ST.
Moreover, SRC satisfies two key consistency properties---tree equivalence to ORC under the length measure (for $p=1$) and Dirac-limit flatness---which support SRC as a principled curvature measure for curvature-driven flows.

\subsection{Background: ORC-based Ricci Flow}

We briefly recall the ORC-based Ricci flow studied in~\citep{DBLP:journals/corr/abs-1907-03993}.
In Riemannian geometry, let $(M,g)$ be a smooth manifold equipped with a Riemannian metric $g=(g_{ij})$,
and let $R_{ij}$ denote the associated Ricci curvature tensor.
Ricci flow evolves the metric as the nonlinear partial differential equation
\begin{equation}
\label{eq:riemannian_ricci_flow}
    \frac{\partial}{\partial t} g_{ij}(t) \;=\; -2\,R_{ij}(t),
\end{equation}
where $g_{ij}(t)$ is a one-parameter family of metrics on $M$ and $R_{ij}(t)$ is the Ricci curvature computed from $g_{ij}(t)$.
Intuitively, Ricci flow performs a curvature-driven deformation of geometry:
regions of positive curvature tend to contract, while regions of negative curvature tend to expand,
thereby smoothing irregularities of the metric.

\paragraph{From manifolds to graphs.}
On graphs, the metric is represented by edge lengths (or weights), and the role of Ricci curvature is played by a discrete curvature assigned to edges.
The ORC formulation computes the ORC $\kappa^{(t)}(x,y)$ for each edge $\langle x,y\rangle$
using localized neighborhood measures $\mu_x,\mu_y$ and a Wasserstein distance.

\paragraph{Discrete Ricci flow.}
A discrete Ricci flow updates weights $w^{(t+1)}(x,y)$ for edge $\langle x,y\rangle$ according to
\begin{equation}\label{eq:orc_ricci_flow_update}
        w^{(t+1)}(x,y) \;=\; \bigl(1-\kappa^{(t)}(x,y)\bigr)\, d^{(t)}(x,y),
\end{equation}
where $d^{(t)}(x,y)$ is the shortest-path distance induced by the current edge weights.
This update expands negatively curved edges and contracts positively curved ones,
and has been used for community detection and graph surgery.
In practice, ORC-based Ricci flow is computationally expensive since ORC requires
solving an OT problem per edge.

\subsection{Sobolev--Ricci Flow}
We now define an analogous flow driven by SRC.
Let $\kappa_{\mathrm{S}_p}^{(t)}(x,y)$ denote the SRC on edge $\langle x,y\rangle$
defined in~\eqref{eq:sobolev_ricci_curvature}, and let $d^{(t)}$ be the shortest-path
metric induced by $w^{(t)}$.
Following~\eqref{eq:orc_ricci_flow_update}, we update edge weights $w_{xy}^{(t+1)}$ by
\begin{equation}\label{eq:sobolev_ricci_flow}
    w_{xy}^{(t+1)} \;=\; \bigl(1-\kappa_{\mathrm{S}_p}^{(t)}(x,y)\bigr)\, d^{(t)}(x,y),
\end{equation}
with initialization $w_{xy}^{(0)} = d^{(0)}(x,y)$.
This update evolves the graph geometry in a curvature-consistent manner,
while replacing the costly Wasserstein computation in ORC by the scalable
Sobolev transport.
We refer to~\eqref{eq:sobolev_ricci_flow} as the \emph{Sobolev--Ricci Flow (SRF)}.

\subsection{Theoretical justification}\label{sec:theoretical_justification}
Our use of SRC in SRF is supported by two consistency principles established in
Section~\ref{sec:properties}.

\textbf{(i) tree equivalence}: on trees endowed with the length measure and for $p=1$,
SRC coincides with ORC (Proposition~\ref{prop:SRC_orc_equivalence}), hence SRF reduces to the standard ORC-based Ricci flow in this canonical setting.

\textbf{(ii) Dirac-limit flatness}: when neighborhood measures collapse to Dirac masses,
SRC vanishes, ensuring that SRF does not introduce spurious curvature in the flat case discussed in Section~\ref{sec:consistency_with_dirac_measures}.
We therefore view SRF as the Ricci flow naturally induced by ST geometry,
with ORC-based Ricci flow recovered as a special case.

\subsection{Computational Complexity}\label{sec:computational_complexity}

Let $n,m,k$ denote the numbers of nodes, edges, and neighbors (i.e., $|\mathrm{supp}(\mu_x)|$, typically on the degree scale).
ORC solves an OT problem per edge; Sinkhorn requires $T_{\mathrm{sk}}$ iterations with $\tilde{O}(k^2)$ updates (polylog suppressed),
yielding $\tilde{O}(T_{\mathrm{sk}}mk^2)$ per Ricci-flow iteration.
SRC replaces per-edge OT solves with closed-form ST on a tree metric.
SRC(SPT) and SRC(MST) require only tree construction and local aggregation over supports of size $k$.
With dense evaluation over $|E(T_r)|=n-1$ tree edges, the per-iteration cost is
$\tilde{O}(m\log n + nk\,h + mn)$; typically $h_{\mathrm{SPT}}<h_{\mathrm{MST}}$, so SRC(SPT) is faster than SRC(MST),
while the comparison to ORC depends on $(n,k,T_{\mathrm{sk}})$.
See Appendix~\ref{sec:complexity_appendix} for derivations.

\section{Numerical experiments}\label{sec:numerical_experiments}

\subsection{Community Detection with Ricci Flow}
\label{sec:community_detection_with_ricciflow}

We evaluate SRC for Ricci-flow-based community detection, comparing it with ORC in accuracy and runtime.
We report results on synthetic and real-world networks.

\subsubsection{Experimental Setup}
\paragraph{Datasets.}
We employ two synthetic benchmarks, \textbf{SBM} and \textbf{LFR},
as well as five widely used real-world networks
(\texttt{karate}, \texttt{football}, \texttt{polbooks}, \texttt{polblogs}, and \texttt{email-eu-core}).
See Appendix~\ref{sec:profiles_of_datasets} for dataset profiles and licenses.

\paragraph{Baselines.}
We compare SRF with the standard Ollivier Ricci flow (ORC; \texttt{GraphRicciCurvature}~\cite{graphriccicurvature}, default settings, Sinkhorn OT).
Our method has two variants: SRC(SPT) and SRC(MST).
As non-flow baselines, we report results from community detection algorithms in \texttt{igraph} library (\url{https://igraph.org}):
spinglass, infomap, fast greedy, edge betweenness, and label propagation\footnote{We do not include JC/FC as flow baselines, as we focus on transport-based metric flows 
.}.
Detailed hyperparameter settings and dataset configurations are reported in Appendix~\ref{app:exp_settings}.

\paragraph{Evaluation Protocol.}
Our evaluation follows a two-step procedure.
First, we run the Ricci flow for a number of iterations, where each iteration recomputes curvature and updates the edge weights $w_{xy}$,
which are interpreted as edge lengths/dissimilarities and induce a shortest-path metric $d^{(t)}(\cdot,\cdot)$ on the graph.
Second, we perform community detection using the Louvain method (Appendix~\ref{app:louvain}) and compute ARI against the ground-truth partition.

Since Louvain and most modularity-based methods operate on \emph{similarity} weights,
we convert the learned Ricci-flow weights $w_{xy}^{(t)}$ into affinity weights via
\begin{equation}
\label{eq:length_to_similarity}
\tilde{w}_{xy}^{(t)} \;=\; \exp(-\beta\, w_{xy}^{(t)}),
\end{equation}
with a fixed $\beta>0$ throughout the experiments.
This transformation is purely to match the input convention of Louvain and does not alter the Ricci flow itself. ($\beta$ plays the role of a temperature parameter for converting length-like quantities into similarities; unless otherwise stated, we fix $\beta = 1$, and Appendix~\ref{sec:ablation_beta} shows that the downstream performance is insensitive to $\beta$ over a broad range.) 
We then apply the Louvain method to $(V,E,\tilde{w})$ and evaluate ARI.
We adopt Louvain as a common post-processing step for evaluation, following standard practice
in community detection benchmarks.
Compared to direct length-based thresholding, modularity-based clustering provides more stable
partitions across datasets and facilitates fair comparison with existing baselines.

We report ARI on all datasets, and measure wall-clock runtime only on SBM,
where graph size and recovery difficulty can be systematically controlled.
For real networks, we report mean±std over trials (SRC(SPT): multiple roots; spinglass/label propagation: non-deterministic algorithms).

\subsubsection{Results}\label{sec:results_on_srf}

On synthetic benchmarks (SBM/LFR), Fig.~\ref{fig:sbm_lfr_real_runtime}(a,b) reports ARI on SBM and LFR while varying the difficulty of community recovery.
Across both benchmarks, the curvature-driven flows (ORC/SRC) achieve near-perfect recovery in the easy regimes and degrade gracefully as the problem becomes harder.
In the transition region, ORC attains the best performance, and SRC(MST) is typically slightly better than SRC(SPT), while both remain competitive.
Across repeated trials, the variability of ORC and SRC remains low for most settings, with increased dispersion mainly observed around the transition region where recovery becomes intrinsically unstable.
Additional ablation studies and curvature statistics are reported in Appendix~\ref{app:ablation_study_on_comm_detection}.
While prior work evaluates ORC-based flows with a direct length-based partitioning~\citep{DBLP:journals/corr/abs-1907-03993}, we report Louvain-based ARI as a standard downstream readout and compare both evaluations in Appendix~\ref{sec:ablation_p_downstream_lfr}, under which the methods remain competitive.
Also, SRC(SPT) is root-dependent; Theorem~\ref{thm:root_dependence_bound} in Appendix~\ref{app:proofs} bounds the curvature variation under changes of the root, and Fig.~\ref{fig:result_sbm_lfr_root_ratio_mean_std} empirically shows that the resulting perturbations are small on both SBM and LFR.

On real networks, Fig.~\ref{fig:sbm_lfr_real_runtime}(c) summarizes results on five real-world networks.
SRC remains competitive with ORC across datasets, and SRC(SPT) often matches ORC in ARI; standard baselines behave consistently with prior observations~\citep{DBLP:journals/corr/abs-1907-03993}.
Since SRC(SPT) can depend on the root, we aggregate results over multiple roots, and the resulting variability is limited, indicating robustness in practice.
On \texttt{karate}, however, both SRC(MST) and SRC(SPT) achieve relatively low ARI compared to other real-world datasets, unlike the stable performance on larger networks.
We further analyze this behavior via an $\alpha$ sensitivity study in Appendix~\ref{app:sensitivity_to_alpha_small_nw}.

On SBM, Fig.~\ref{fig:sbm_lfr_real_runtime}(d) reports runtime: SRC(SPT) is the fastest, followed by SRC(MST), while ORC is $10$--$100\times$ slower.
This trend is in line with Appendix~\ref{subsec:complexity_sbm_numbers}, which suggests that ORC incurs a higher per-edge cost due to OT computations.

\begin{figure*}[t]
    \centering
    \begin{subfigure}[t]{0.48\textwidth}
        \centering
        \includegraphics[width=\linewidth]{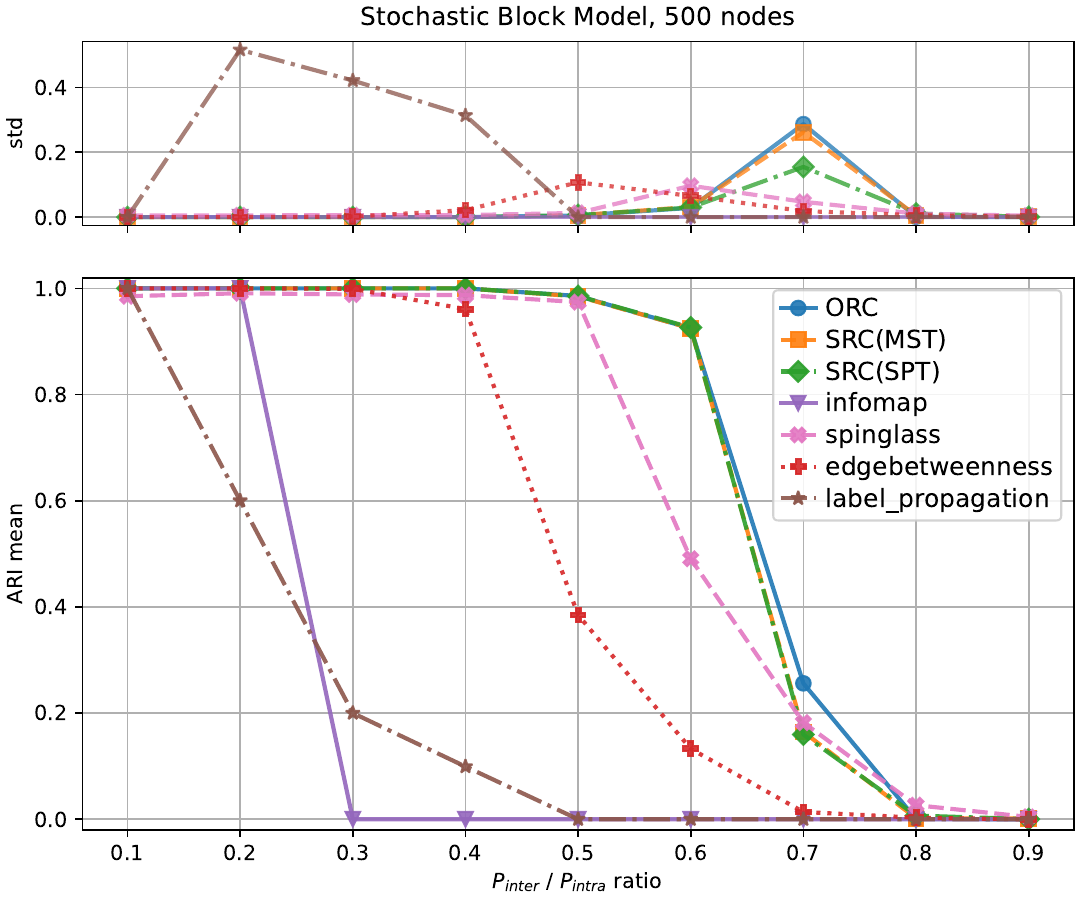}
        \caption{SBM}
        \label{fig:sbm}
    \end{subfigure}
    \hfill
    \begin{subfigure}[t]{0.48\textwidth}
        \centering
        \includegraphics[width=\linewidth]{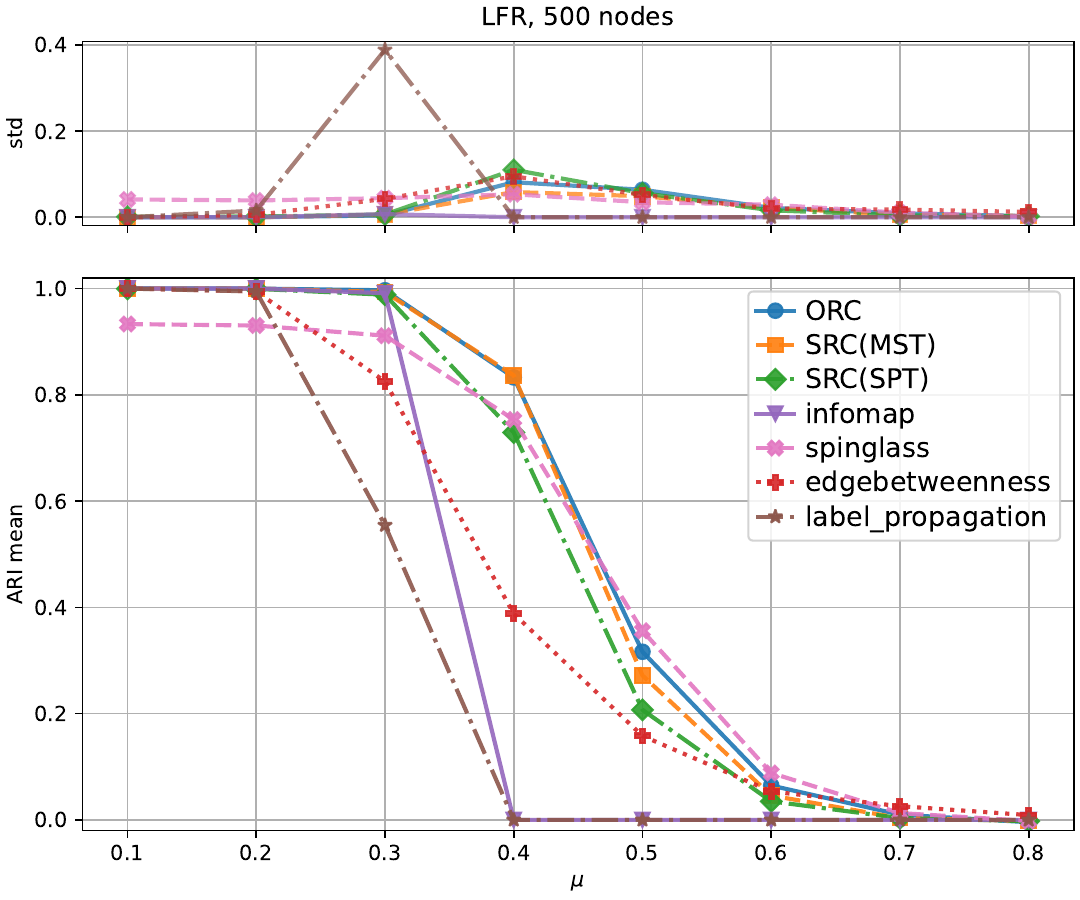}
        \caption{LFR}
        \label{fig:lfr}
    \end{subfigure}

    \vspace{2mm}

    \begin{subfigure}[t]{0.48\textwidth}
        \centering
        \includegraphics[width=\linewidth]{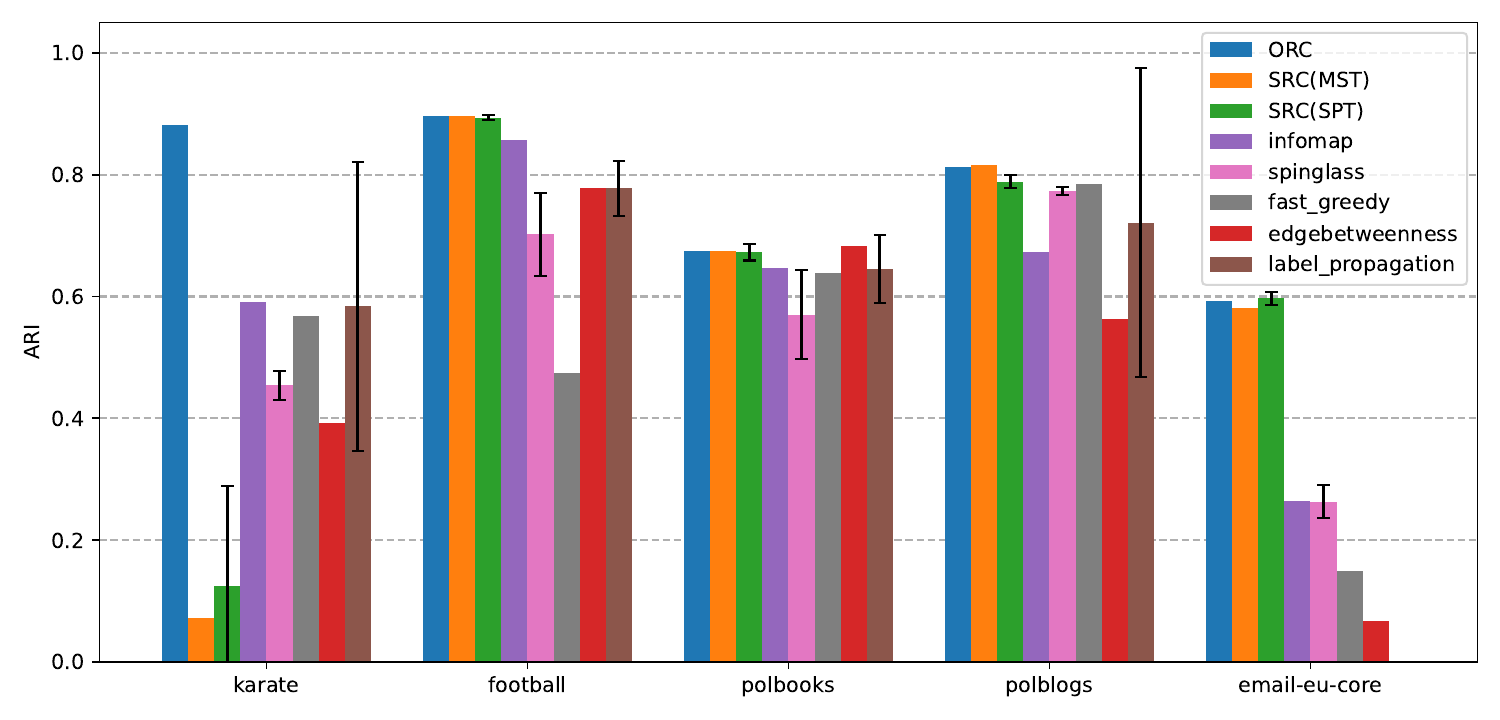}
        \caption{Real Networks}
        \label{fig:real_summary}
    \end{subfigure}
    \hfill
    \begin{subfigure}[t]{0.48\textwidth}
        \centering
        \includegraphics[width=\linewidth]{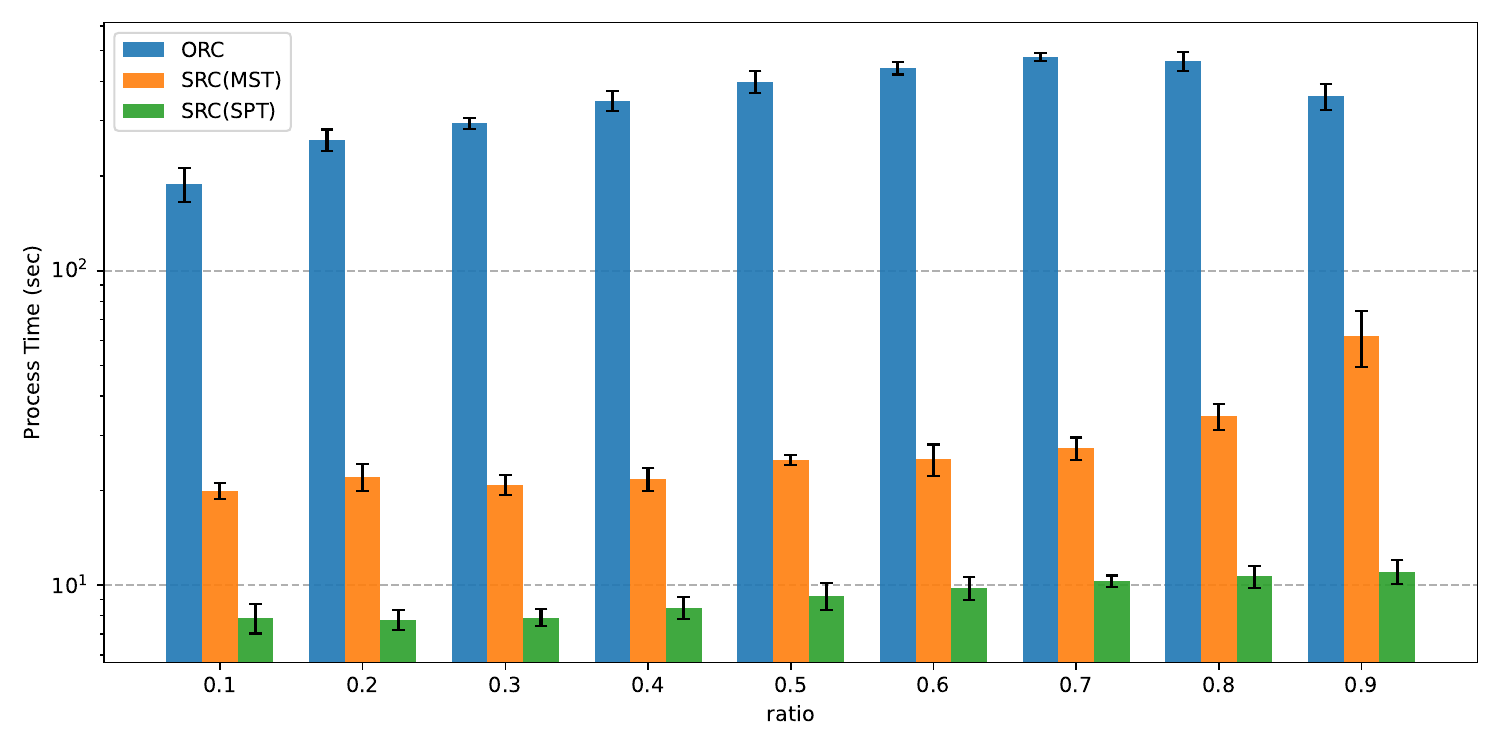}
        \caption{SBM Runtime}
        \label{fig:sbm_runtime}
    \end{subfigure}

    \vspace{-2mm}
    \caption{Community detection performance (ARI) and computational cost under Ricci flow.}
    \vspace{-2mm}
    \label{fig:sbm_lfr_real_runtime}
\end{figure*}

\subsection{Pruning edges with Curvatures}\label{sec:edge_pruning_with_Curvature}

We address the edge pruning problem in nearest-neighbor graphs.
When graphs are built from noisy samples on a manifold, they often contain shortcut edges that connect distant parts of the manifold through the ambient space. Such edges distort geodesic distances and degrade downstream tasks.

The goal of pruning is therefore to remove these shortcuts while preserving good edges that reflect the intrinsic manifold structure.
To evaluate pruning quality, we measure the proportion of correctly removed shortcut edges (true positives) and the proportion of mistakenly removed good edges (false positives). High pruning accuracy means that the curvature-based method successfully separates shortcuts from desirable edges, thus serving as a test of whether the curvature captures meaningful geometric information.

We conduct edge pruning experiments across ten benchmark datasets
(e.g., \texttt{concentric\_circles}, \texttt{torii}, \texttt{s\_curve}, and \texttt{3D\_swiss\_roll}),
each consisting of $4{,}000$ data points.
The full list and visualizations are provided in Appendix~\ref{sec:visualization_of_datasets_for_edge_pruning}.

Following the two-stage pruning strategy of ORC-MANL (Manifold Learning and Recovery)~\citep{DBLP:conf/iclr/SaidiHB25}, we alternate curvature-based edge filtering with shortcut removal.
In the first stage, candidate edges are selected based on curvature values using threshold parameter $\delta_{\text{M}}$.
In the second stage, we remove shortcut edges if their shortest-path distance violates a threshold determined by a parameter  $\lambda_{\text{M}}$.

\begin{itemize}
\item \textbf{ORC(lightweight)}: simplified ORC on a $k$NN graph, consistent with ORC-MANL.
\item \textbf{SRC(SPT)}: a variant of SRC based on an SPT, designed to capture local neighborhood structure.
\end{itemize}
It should be noted that while ORC and SRC can be directly compared, JC and FC cannot be directly compared, since their consistency with the Dirac measure is unclear and the ranges of curvature values they take differ substantially.
We perform pruning using the following five criteria:

\textbf{Distance-only}: removes edges exceeding a length threshold without curvature information.

\textbf{ORC only}: removes edges with $\kappa_{\mathrm{ORC}} \le -1 + 4(1-\delta_{\text{M}})$. See~\citep[\S A.3.1]{DBLP:conf/iclr/SaidiHB25} for details.

\textbf{ORC--MANL}: applies two-stage pruning using ORC in the first stage and shortcut detection in the second.

\textbf{SRC only}: same as ``ORC only'' but with SRC.

\textbf{SRC--MANL}: applies two-stage pruning using SRC in the first stage.

We set $(\delta_{\text{M}},\lambda_{\text{M}})=(0.75, 0.01)$, 
which achieve the best scores for both ORC-MANL and SRC-MANL 
within the parameter ranges 
$\delta_{\text{M}} \in [0.7,0.9]$ (step 0.05) and 
$\lambda_{\text{M}} \in \{0.001,0.01,0.1,0.2,0.5\}$, whose combinations are followed in~\citep{DBLP:conf/iclr/SaidiHB25}. 

Figure~\ref{fig:results_edge_pruning_selected} illustrates representative pruning outcomes, highlighting the benefit of curvature information.
\emph{Distance-only} pruning fails to reliably separate shortcut edges, underscoring the need for curvature information.
Comparing curvature signals, \emph{SRC only} consistently detects shortcut edges while avoiding excessive removal of good edges, whereas \emph{ORC only} often misclassifies edges.

\begin{wrapfigure}{r}{0.55\linewidth}
    \centering
    \includegraphics[width=\linewidth]{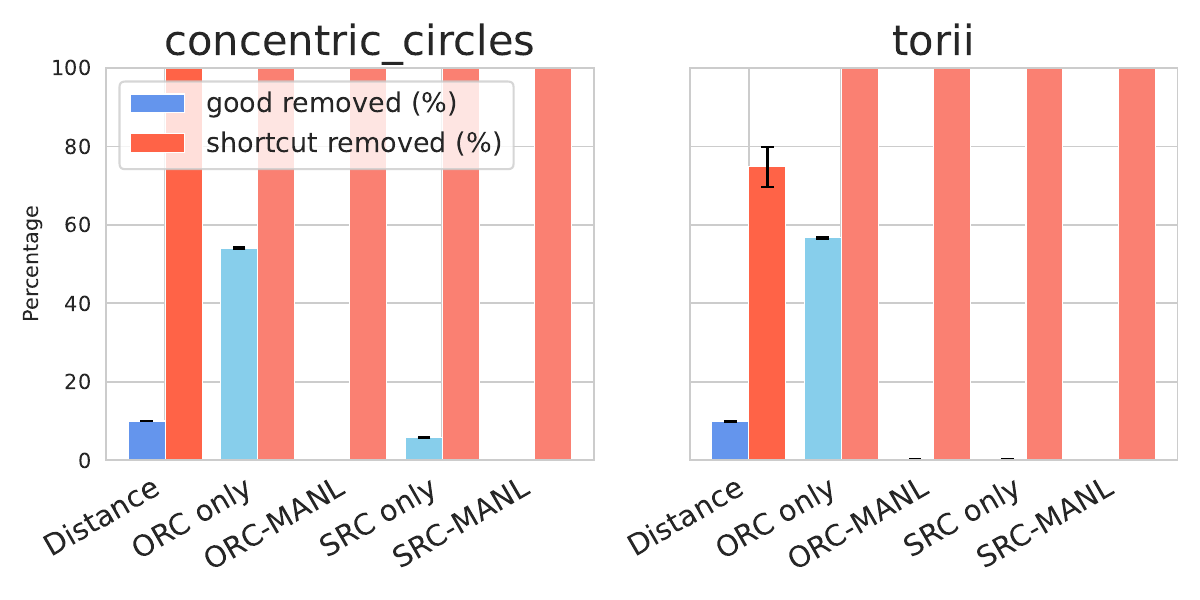}
    \caption{The bars show the percentage of shortcut edges correctly removed (red) and the percentage of good edges incorrectly removed (light blue). Representative results on two datasets.}
    \label{fig:results_edge_pruning_selected}
\end{wrapfigure}

Extending to the MANL framework, both SRC-MANL and ORC-MANL achieve accurate pruning.
SRC-MANL benefits from precise pruning already at the first stage, while ORC-MANL compensates for its initial lag through subsequent MANL operations.

These findings highlight that SRC(SPT) is both computationally efficient and substantially more reliable 
than the $k$NN-based ORC approximation for edge pruning. More results are provided 
in Appendix~\ref{sec:edge_pruning_for_all}. 

We conjecture that this improvement arises because of systematic differences in the curvature values assigned by SRC and ORC. Further intuition and an illustrative example are provided in Appendix~\ref{sec:intuition_on_curvature_values}.

\section{Conclusion and Discussion}\label{sec:conclusion}

In this work, we introduce \emph{Sobolev--Ricci Curvature (SRC)}, a Ricci curvature on graphs
naturally induced by Sobolev transport geometry.
SRC admits efficient evaluation through tree-metric Sobolev structure, and we instantiate it via
SRC(MST) and SRC(SPT).
We demonstrate SRC as a general-purpose curvature primitive in two representative pipelines:
SRF for community detection and SRC--MANL for manifold-oriented edge pruning,
where SRC matches ORC-based pipelines while offering substantial computational savings.
Moreover, our theory shows that ORC is recovered as a special case of SRC on trees endowed with the length measure (for $p=1$),
and SRC is flatness-consistent in the Dirac limit, providing a principled foundation for scalable curvature-driven graph transformations.
\paragraph{Limitations and Future Work.}
SRC depends on the tree choice, reflecting a modeling decision on the underlying Sobolev geometry; the SRC distribution is stable across tree constructions (See Appendix~\ref{app:tree_robustness}).
While we considered MST and SPT instantiations, developing principled and data-driven strategies for selecting or learning this structure remains an open direction.
Another limitation is that our formulation targets static graphs; extending SRC to dynamic or temporal networks will likely require time-dependent Sobolev structures and corresponding curvature notions.

\newpage

\section*{Impact Statement}
This paper presents work whose goal is to advance the field of machine learning.
There are many potential societal consequences of our work, none of which we feel
must be specifically highlighted here.

\bibliographystyle{apalike}
\bibliography{TWC}

\newpage
\appendix
\onecolumn

\renewcommand{\theequation}{\thesection.\arabic{equation}}
\setcounter{equation}{0}

\section{Proofs}\label{app:proofs}

This appendix provides detailed proofs for the main results, along with brief theoretical remarks
on root dependence and Sobolev geometry under the default choice $p=1$.

\subsection{Proof of Proposition~\ref{prop:SRC_orc_equivalence}}\label{subsec:SRC_orc_equivalence}

\begin{proof}
By \citet[Corollary 4.3]{DBLP:conf/aistats/LeNPN22}, the ST distance $S_1$ 
coincides with the Wasserstein distance $W_1$ when $G$ is a tree 
and $\lambda^*$ is the length measure. 
Since the ORC is defined in~\eqref{eq:orc_def}, and our SRC $\kappa_{\mathrm{S_1}}$ is defined analogously by replacing $W_1$ with $S_1$ as in~\eqref{eq:sobolev_ricci_curvature}. 
The conclusion follows directly:
\begin{equation}
\kappa_{\mathrm{S_1}}(x,y)
  = 1 - \frac{S_1(\mu_x,\mu_y)}{S_1(\delta_x,\delta_y)}
  = 1 - \frac{W_1(\mu_x,\mu_y)}{d(x,y)} \\
  = \kappa_{\mathrm{ORC}}(x,y). 
\end{equation}
\end{proof}

\subsection{Root dependence bound for SRC}\label{app:root_dependence_bount_for_src}

This section studies the dependence of SRC on the choice of root node.
We derive an explicit bound showing that, under mild assumptions,
the resulting curvature values vary in a controlled manner,
thereby justifying the robustness of SRC with respect to root selection.

\begin{theorem}[Root-dependence bound for SRC, $p=1$]
\label{thm:root_dependence_bound}
Let $G=(V,E,\ell)$ be a connected undirected graph with positive edge lengths $\ell(e)>0$.
For a choice of root $r\in V$, let $T_r$ denote the rooted tree used by SRC(SPT),
and define the SRC with $p=1$ on an edge $\langle u,v\rangle\in E$ by
\begin{equation}
\label{eq:src_curvature_tree_form}
\kappa^{(r)}(u,v)
\;:=\;
1-\frac{S_r(u,v)}{D_r(u,v)},
\end{equation}
where $S_r(u,v)$ is the tree-based Sobolev transport cost between $\mu_u$ and $\mu_v$ on $T_r$,
and $D_r(u,v)=d_{T_r}(u,v)$ is the induced tree distance.

Consider another root $r'\in V$ with the corresponding tree $T_{r'}$.
Let the symmetric difference of tree edge sets be
\begin{equation}
\label{eq:delta_definition}
\Delta \;:=\; E(T_r)\,\triangle\,E(T_{r'}).
\end{equation}
Assume that there exist constants $\ell_{\max}>0$, $D_{\min}>0$, and $S_{\max}>0$ such that
\begin{align}
\label{eq:assump_lmax}
&\ell(e)\le \ell_{\max} \qquad (\forall e\in E),\\
\label{eq:assump_dmin}
&D_r(u,v)\ge D_{\min},\quad D_{r'}(u,v)\ge D_{\min},\\
\label{eq:assump_smax}
&S_r(u,v)\le S_{\max},\quad S_{r'}(u,v)\le S_{\max}.
\end{align}
Then for any $\langle u,v\rangle\in E$, the SRC satisfies
\begin{equation}
\label{eq:root_dependence_bound}
\bigl|\kappa^{(r)}(u,v)-\kappa^{(r')}(u,v)\bigr|
\;\le\;
\ell_{\max}\,|\Delta|\left(
\frac{1}{D_{\min}}
+
\frac{S_{\max}}{D_{\min}^2}
\right).
\end{equation}
\end{theorem}

\begin{proof}
Fix an edge $\langle u,v\rangle\in E$.
By definition,
\[
\kappa^{(r)}(u,v)=1-\frac{S_r(u,v)}{D_r(u,v)},\qquad
\kappa^{(r')}(u,v)=1-\frac{S_{r'}(u,v)}{D_{r'}(u,v)}.
\]
Hence
\begin{equation}
\label{eq:kappa_diff_reduce}
\bigl|\kappa^{(r)}(u,v)-\kappa^{(r')}(u,v)\bigr|
=
\left|
\frac{S_r(u,v)}{D_r(u,v)}-\frac{S_{r'}(u,v)}{D_{r'}(u,v)}
\right|.
\end{equation}

\paragraph{Step 1: A generic ratio perturbation bound.}
For any positive numbers $S,S',D,D'>0$,
\[
\left|\frac{S}{D}-\frac{S'}{D'}\right|
=
\left|\frac{S-S'}{D} + S'\left(\frac{1}{D}-\frac{1}{D'}\right)\right|
\le
\frac{|S-S'|}{D} + \frac{S'}{DD'}|D-D'|.
\]
Applying this with $(S,D)=(S_r(u,v),D_r(u,v))$ and $(S',D')=(S_{r'}(u,v),D_{r'}(u,v))$ gives
\begin{equation}
\label{eq:ratio_bound_initial}
\left|
\frac{S_r(u,v)}{D_r(u,v)}-\frac{S_{r'}(u,v)}{D_{r'}(u,v)}
\right|
\le
\frac{|S_r(u,v)-S_{r'}(u,v)|}{D_r(u,v)}
+
\frac{S_{r'}(u,v)}{D_r(u,v)\,D_{r'}(u,v)}\,|D_r(u,v)-D_{r'}(u,v)|.
\end{equation}

\paragraph{Step 2: Bounding changes in tree distance.}
Recall that $D_r(u,v)=d_{T_r}(u,v)$ is the path length between $u$ and $v$ on $T_r$.
Let $P_r(u,v)$ denote the (unique) path between $u$ and $v$ in $T_r$,
and similarly $P_{r'}(u,v)$ in $T_{r'}$.
Since both are trees, the distance is a sum of edge lengths along the path:
\[
D_r(u,v)=\sum_{e\in P_r(u,v)} \ell(e),\qquad
D_{r'}(u,v)=\sum_{e\in P_{r'}(u,v)} \ell(e).
\]
Let $\Delta = E(T_r)\triangle E(T_{r'})$ be the symmetric difference of tree edge sets.
Only edges in $\Delta$ can differ between the two paths.
Thus
\[
|D_r(u,v)-D_{r'}(u,v)|
\le
\sum_{e\in P_r(u,v)\triangle P_{r'}(u,v)} \ell(e)
\le
\sum_{e\in \Delta} \ell(e)
\le
\ell_{\max}\,|\Delta|,
\]
where the last inequality uses the uniform bound $\ell(e)\le \ell_{\max}$ from~\eqref{eq:assump_lmax}.

\paragraph{Step 3: Bounding changes in Sobolev transport cost.}
The quantity $S_r(u,v)$ is a tree-based Sobolev transport cost on $T_r$,
which can be written as a sum over tree edges (cf. Eq.~(E.30) in the appendix):
\[
S_r(u,v) \;=\; \sum_{e\in E(T_r)} \lambda_r(e)\,\bigl|F_u^{(r)}(e)-F_v^{(r)}(e)\bigr|,
\]
where $\lambda_r(e)$ is the length associated with tree edge $e$ and $F_u^{(r)}(e)$ is the aggregated flow value.
Since the flow difference term is bounded by $1$ (both $F_u^{(r)}(e)$ and $F_v^{(r)}(e)$ represent total masses),
we obtain
\[
|S_r(u,v)-S_{r'}(u,v)|
\le
\sum_{e\in E(T_r)\triangle E(T_{r'})} \lambda(e)
\le
\sum_{e\in \Delta} \ell(e)
\le
\ell_{\max}\,|\Delta|.
\]
Here we used again $\ell(e)\le \ell_{\max}$ and that only edges in $\Delta$ contribute to the mismatch of the two tree-sums.

\paragraph{Step 4: Combine bounds under uniform lower/upper assumptions.}
Plugging the bounds from Steps 2--3 into~\eqref{eq:ratio_bound_initial} and using the assumptions
$D_r(u,v)\ge D_{\min}$, $D_{r'}(u,v)\ge D_{\min}$ from~\eqref{eq:assump_dmin},
and $S_{r'}(u,v)\le S_{\max}$ from~\eqref{eq:assump_smax}, we obtain
\[
\left|
\frac{S_r(u,v)}{D_r(u,v)}-\frac{S_{r'}(u,v)}{D_{r'}(u,v)}
\right|
\le
\frac{\ell_{\max}|\Delta|}{D_{\min}}
+
\frac{S_{\max}}{D_{\min}^2}\,\ell_{\max}|\Delta|.
\]
Together with~\eqref{eq:kappa_diff_reduce}, this implies
\[
\bigl|\kappa^{(r)}(u,v)-\kappa^{(r')}(u,v)\bigr|
\le
\ell_{\max}\,|\Delta|
\left(
\frac{1}{D_{\min}} + \frac{S_{\max}}{D_{\min}^2}
\right),
\]
which is exactly~\eqref{eq:root_dependence_bound}.
\end{proof}

\paragraph{From a pointwise bound to an averaged diagnostic.}
We recall Theorem~A.1, which bounds the root dependence of the SRC on each edge.
For two roots $r,r'\in V$, define
\[
\Delta := E(T_r)\triangle E(T_{r'})
\]
as the symmetric difference of the tree edge sets.
Under Assumptions~\eqref{eq:assump_lmax} -~\eqref{eq:assump_smax}, Theorem~A.1 yields that for any edge $e=\langle u,v\rangle\in E$,
\begin{equation}\label{eq:pointwise_root_bound}
\bigl|\kappa^{(r)}(e)-\kappa^{(r')}(e)\bigr|
\;\le\;
\ell_{\max}|\Delta|
\left(\frac{1}{D_{\min}}+\frac{S_{\max}}{D_{\min}^2}\right)
=: C\,|\Delta|,
\end{equation}
where $C := \ell_{\max}\bigl(\frac{1}{D_{\min}}+\frac{S_{\max}}{D_{\min}^2}\bigr)$.

\paragraph{An $L^1$-type root sensitivity measure.}
To summarize the overall root dependence across the entire graph,
we consider the edgewise $L^1$ difference
\[
\|\Delta\kappa\|_1
\;:=\;
\frac{1}{|E|}
\sum_{e\in E}
\bigl|\kappa^{(r)}(e)-\kappa^{(r')}(e)\bigr|.
\]
Combining this with~\eqref{eq:pointwise_root_bound}, we obtain
\begin{equation}\label{eq:l1_bound}
\|\Delta\kappa\|_1
\;\le\;
C\,|\Delta|.
\end{equation}

\paragraph{Normalized diagnostic (per changed tree edge).}
Dividing both sides of~\eqref{eq:l1_bound} by $|\Delta|$ (assuming $|\Delta|>0$) gives
\begin{equation}\label{eq:ratio_bound}
\frac{\|\Delta\kappa\|_1}{|\Delta|}
\;\le\;
C
\;=\;
\ell_{\max}\!\left(\frac{1}{D_{\min}}+\frac{S_{\max}}{D_{\min}^2}\right).
\end{equation}
The left-hand side corresponds exactly to our empirical diagnostic,
which can be interpreted as the \emph{average absolute curvature perturbation per changed tree edge}.

\paragraph{Expected bound under random root selection.}
If $r,r'$ are sampled uniformly at random from $V$, then taking expectation in~\eqref{eq:ratio_bound} yields
\begin{equation}\label{eq:expected_ratio_bound}
\mathbb{E}_{r,r'}\!\left[\frac{\|\Delta\kappa\|_1}{|\Delta|}\right]
\;\le\;
\ell_{\max}\!\left(\frac{1}{D_{\min}}+\frac{S_{\max}}{D_{\min}^2}\right),
\end{equation}
provided that $D_{\min}$ and $S_{\max}$ hold uniformly over the sampled roots.

\paragraph{Connection to SBM/LFR diagnostics.}
Figure~\ref{fig:result_sbm_lfr_root_ratio_mean_std} visualizes the empirical quantity
$\mathbb{E}_{r,r'}\bigl[\|\Delta\kappa\|_1/|\Delta|\bigr]$ for SRC(SPT),
estimated by sampling multiple root pairs $(r,r')$ and averaging over repeated graph instances.
Here, $\|\Delta\kappa\|_1$ is the $L^1$ difference of edge-wise SRC under two roots,
and $|\Delta|$ is the number of tree edges that differ between the corresponding rooted trees,
i.e., $\Delta = E(T_r)\triangle E(T_{r'})$.
Thus the plotted value can be interpreted as the \emph{average absolute curvature change per changed tree edge}.
The observed ratios remain small across both SBM and LFR,
supporting the bound in~\eqref{eq:expected_ratio_bound} and indicating that
the root dependence of SRC(SPT) is limited in practice.

\begin{figure}
    \centering
    \includegraphics[width=0.9\linewidth]{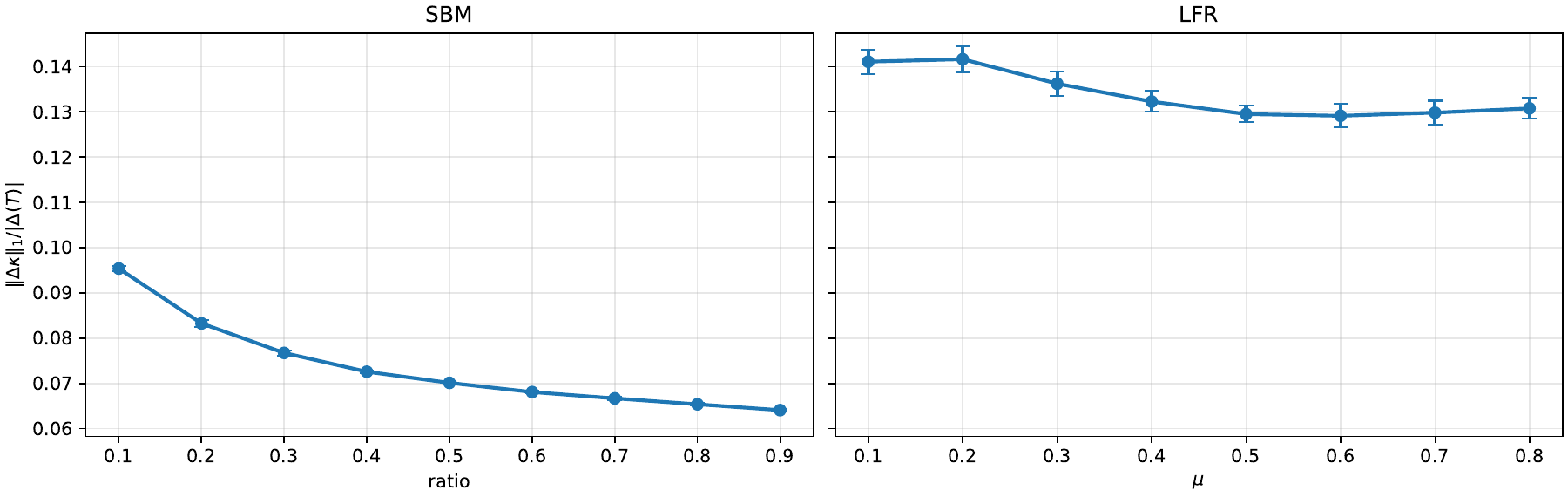}
    \caption{Empirical root sensitivity of SRC(SPT) on SBM and LFR.
For each dataset, we estimate $\mathbb{E}_{r,r'}[\|\Delta\kappa\|_1/|\Delta|]$
by sampling multiple root pairs $(r,r')$ and averaging over repeated graph instances,
where $\|\Delta\kappa\|_1$ is the $L^1$ difference of edge curvatures and
$|\Delta|=|E(T_r)\triangle E(T_{r'})|$ is the number of tree edges that change.
This quantity corresponds to the left-hand side of~\eqref{eq:expected_ratio_bound}
and measures the average curvature perturbation per changed tree edge.}
    \label{fig:result_sbm_lfr_root_ratio_mean_std}
\end{figure}

\subsection{Interpretation of Sobolev geometry}
\label{app:sobolev-geom}

In Appendix~\ref{app:root_dependence_bount_for_src}, we established that the curvature induced by the ST
depends on the choice of root (and the corresponding extracted tree), but that this
dependence is quantitatively controlled by the mismatch of tree edges.
In this subsection, we provide a geometric interpretation of this behavior,
without making any stronger invariance claims.
Our goal is to clarify why the normalization appearing in~\eqref{eq:kappa_diff_reduce} is geometrically meaningful,
rather than to define a fully root-independent geometry.

\subsubsection{Notation (recap)}
Let $G=(V,E_G)$ be a weighted connected graph, and let $z_0\in V$ be a designated root.

We identify $\mathbb{R}^{V}$ (resp.\ $\mathbb{R}^{E}$) with the space of real-valued
functions on the vertex set $V$ (resp.\ the tree-edge set $E$).

Let $T_{z_0}=(V,E)$ denote an extracted tree rooted at $z_0$ (e.g., a shortest-path tree).
For each tree edge $e\in E$, removing $e$ splits $T_{z_0}$ into two connected components;
we denote by $T_e\subset V$ the component not containing the root $z_0$, and define the
corresponding shadow set $\gamma_e := \mathbf{1}_{T_e}\in\{0,1\}^{V}$.
We write $\mathcal{P}(V):=\{\mu\in\mathbb{R}^V_+:\sum_{v\in V}\mu(v)=1\}$ for the probability simplex.
Throughout this subsection, the tree-edge weights $\lambda(e)>0$ are inherited from the
original graph weights restricted to the extracted tree.

\subsubsection{Correspondence with continuous geometry}

Table~\ref{table:correspondence_table} summarizes the correspondence between standard
notions in continuous Riemannian/Finsler geometry and the discrete Sobolev geometry
considered in this paper.
This table is intended purely as an interpretative guide:
while the structures are analogous, no claim is made that the discrete geometry
arises as a limit of a smooth manifold.

\paragraph{Why points are measures (intuition).}
Our Sobolev geometry is defined on the probability simplex $\mathcal{P}(V)$, i.e., distributions on the vertex set $V$.
Hence a ``point'' in this geometry is a measure $\mu\in\mathcal{P}(V)$ (equivalently, a nonnegative function on $V$ summing to one),
and Sobolev transport compares such measures.
The root $z_0$ does not change the point $\mu$ itself; it only changes its coordinate representation through cut masses $\mu(\gamma_e)$.

\begin{table}[t]
\centering
\caption{Interpretative correspondence between continuous geometry and the discrete
Sobolev geometry induced by an extracted tree.}
\small
\begin{tabular}{p{0.48\linewidth} p{0.48\linewidth}}
\toprule
Continuous geometry (Riemann/Finsler) & Discrete Sobolev geometry (this paper)\\
\midrule
Point $x\in M$ & Measure $\mu\in \mathcal{P}(V)$\\

Tangent space $T_xM$ & $T_\mu \mathcal P(V)=
\left\{
\xi \in \mathbb{R}^V \; ;\; \sum_{v\in V}\xi(v)=0
\right\}
\quad (\text{independent of } \mu).
$\\
Chart $\varphi:U\to\mathbb{R}^d$ & Cut-coordinate map $\Phi_{z_0}:\mu\mapsto (\mu(\gamma_e))_{e\in E}$\\
Riemannian metric $g$ & $p=2$: weighted inner product $\langle a,b\rangle_\lambda=\sum_e \lambda(e)a_eb_e$\\
Finsler norm $F(\xi)$ & General $p$: $F(\xi)=\|d\Phi_{z_0}(\xi)\|_{\ell^p_\lambda}$\\
Geodesic distance & $S_p(\mu,\nu)=\|\Phi_{z_0}(\mu)-\Phi_{z_0}(\nu)\|_{\ell^p_\lambda}$\\
Change of chart & Change of extracted tree / root\\
\bottomrule
\end{tabular}
\vspace{2mm}

\label{table:correspondence_table}
\end{table}

\subsubsection{Cut-coordinate embedding viewpoint}

Given a fixed root $z_0$, define the cut-coordinate embedding
$\Phi_{z_0}:\mathcal{P}(V)\to\mathbb{R}^{E}$ by
\[
(\Phi_{z_0}(\mu))_e := \mu(\gamma_e)
= \sum_{v\in V}\mu(v)\gamma_e(v), \qquad e\in E.
\]
We equip $\mathbb{R}^{E}$ with the weighted $\ell^p$ norm
\[
\|a\|_{\ell^p_\lambda(E)} := \Big(\sum_{e\in E}\lambda(e)|a_e|^p\Big)^{1/p}.
\]
With these definitions, the Sobolev transport distance admits the closed form
\[
S_p(\mu,\nu)=\|\Phi_{z_0}(\mu)-\Phi_{z_0}(\nu)\|_{\ell^p_\lambda(E)},
\]
as shown in Appendix~C.1.
This expression motivates viewing $S_p$ as the distance induced by measuring differences
between probability measures in cut coordinates.

\subsubsection{Geometric meaning of the normalization $S/D$}

Recall that the curvature used in this paper is defined through the normalized quantity
$S_{z_0}(u,v)/D_{z_0}(u,v)$.
Here, $S_{z_0}(u,v)$ measures the discrepancy between the local neighborhood measures
around $u$ and $v$ in the Sobolev transport geometry associated with the extracted tree,
while $D_{z_0}(u,v)$ provides the corresponding tree-induced ground distance between the
two base points.

As a result, the ratio $S_{z_0}(u,v)/D_{z_0}(u,v)$ is a dimensionless quantity that can be
interpreted as a local distortion factor of the transport geometry relative to the
underlying tree-metric.
Equation~(A.8),
\[
\bigl|\kappa^{(r)}(u,v)-\kappa^{(r')}(u,v)\bigr|
=
\left|
\frac{S_r(u,v)}{D_r(u,v)}-\frac{S_{r'}(u,v)}{D_{r'}(u,v)}
\right|,
\]
therefore compares two such normalized distortions obtained from different choices of
root (and hence different extracted trees).

Importantly, we do not claim that these quantities are invariant under changes of the
extracted tree.
Rather, Appendix~A.2 establishes that their discrepancy is controlled by the mismatch
between the corresponding tree structures, providing a stability guarantee without
requiring exact geometric equivalence.

\paragraph{Connection to the stability bound in Appendix~\ref{app:root_dependence_bount_for_src}.}
From the correspondence viewpoint in Table~\ref{table:correspondence_table}, changing the root/extracted tree can be viewed as changing the cut-coordinate chart used to represent measures and distances.
Equation~\eqref{eq:kappa_diff_reduce} therefore measures how the \emph{normalized} quantity $S/D$ varies across such chart choices for the same pair $(u,v)$.
Appendix~A.2 complements this interpretation by showing that this variation is controlled by the mismatch between the corresponding extracted trees, yielding a stability guarantee without asserting exact invariance.

\section{From Riemannian Ricci Curvature to Ollivier--Ricci Curvature on Graphs}\label{sec:from_riemannian_rc_to_orc_on_graphs}

In this section we recall how Ricci curvature, originally defined in the smooth setting of Riemannian manifolds, admits reformulations that naturally extend to discrete spaces.

We first review the classical notion of Ricci curvature and its characterization via OT. 
We then explain the small-ball contraction principle, which motivates Ollivier’s coarse Ricci curvature for general metric spaces. 
Specializing this construction to graphs yields a tractable notion of curvature for networks. 
Finally, we highlight the consistency between the discrete and smooth notions, showing that Ollivier curvature converges to classical Ricci curvature in the infinitesimal limit.

\subsection{Ricci curvature on Riemannian manifolds}\label{sec:ricci_curv+on_riem_mfd}
We begin with the classical smooth setting.  
Let $(M,g)$ be an $n$-dimensional Riemannian manifold.  
For $x\in M$ and a unit tangent vector $v\in T_xM$, the \emph{directional Ricci curvature} is defined as the average of sectional curvatures over planes containing $v$:
\begin{equation}
\mathrm{Ric}(v)\;=\;\frac{1}{n-1}\sum_{i=2}^n K_{\mathrm{sec}}(v,v_i),
\end{equation}
where $\{v,v_2,\ldots,v_n\}$ is an orthonormal basis of $T_xM$ with respect to the metric $g_x$, and $K_{\mathrm{sec}}(\cdot,\cdot)$ denotes sectional curvature.

Ricci curvature governs how volumes contract or expand along geodesic flow.  
A modern reformulation links lower Ricci bounds to optimal transport: the works of Lott–Villani and Sturm \cite{lott2009ricci,sturm2006geometry} show that a uniform bound 
\begin{equation}
\mathrm{Ric}(v)\;\ge\; K_0, \qquad \forall v\in T_xM, \text{ and } g_x(v,v)=1,    
\end{equation}

for some constant $K_0\in\mathbb{R}$,
is equivalent to the Boltzmann entropy being $K_0$--convex along $2$--Wasserstein geodesics $(\rho_t)_{t\in[0,1]}$:
\begin{equation}
S(\rho_t)\ \ge\ (1-t)S(\rho_0)+tS(\rho_1)+\tfrac{K_0}{2}t(1-t)\,W_2(\rho_0,\rho_1)^2,
\end{equation}
with entropy $S(\rho)=-\int_M \rho \log \rho\,dm$.  
Intuitively, positive Ricci curvature enforces an ``average contraction'' of mass distributions as they evolve under optimal transport.

\subsection{Ball-contraction inequality}
An equivalent geometric picture is obtained by comparing small geodesic balls.  
Let $m_{r,x}$ be the uniform probability measure on the ball $B(x,r)$. Then as $r\to 0$,
\begin{equation}\label{eq:ball-contraction}
W_1(m_{r,x},m_{r,y}) \;\le\; 
\Bigl(1 - \tfrac{\mathrm{Ric}(v)}{2(n+2)}\,r^2 + o(r^2)\Bigr)\, d(x,y),
\end{equation}
where $v$ is the unit tangent from $x$ to $y$.  
Thus, \emph{positive Ricci curvature implies that the $W_1$ distance between small balls shrinks more than the center distance $d(x,y)$}.  
This inequality will serve as the prototype for the discrete definition.

\subsection{Ollivier curvature on metric spaces}
Ollivier \cite{ollivier2009ricci} generalized the ball-contraction idea from smooth manifolds to arbitrary metric measure spaces.  
Let $(X,d)$ be a metric space equipped with a Markov kernel.  
Associate to each $x\in X$ a probability measure $p_x$ (e.g., the distribution of one step of a random walk).  
Then the \emph{Ollivier--Ricci curvature (ORC)} is defined by
\begin{equation}\label{eq:ORC}
W_1(p_x,p_y) \;=\; \bigl(1-\kappa(x,y)\bigr)\,d(x,y), \text{ for } x,y\in X.
\end{equation}
Here $W_1$ is the $1$--Wasserstein distance on $X$.  
This definition mirrors \eqref{eq:ball-contraction}, but replaces infinitesimal balls $m_{r,x}$ by neighborhood measures $p_x$. Based on the equation~\eqref{eq:ORC}, the sign of the curvature $\kappa$ can be interpreted in terms of the relationship between $W_1$ and $d$ as follows:
\begin{itemize}
\item $\kappa(x,y)>0$: 
\begin{equation}
W_1(p_x,p_y) < d(x,y),    
\end{equation}

meaning that the neighborhoods $p_x$ and $p_y$ are closer in $W_1$ than their centers in $d$ (contraction, positive curvature).
\item $\kappa(x,y)<0$: 
\begin{equation}
W_1(p_x,p_y) > d(x,y),    
\end{equation}
so the neighborhoods spread farther apart than their centers (expansion, negative curvature).
\end{itemize}

\subsection{Specialization to graphs}
Graphs are a natural discrete setting of the above.  
Let $G=(V,E,w)$ be a weighted graph with degree $d_x=\sum_{z\in \mathcal{N}(x)}w_{xz}$, where $\mathcal{N}(x)$ is the set of $k$-nearest neighbors.
The one-step random walk measure is
\begin{equation}
p_x(y)=\frac{w_{xy}}{d_x},\quad y\in \mathcal{N}(x).
\end{equation}
Using \eqref{eq:ORC}, the edge curvature $\kappa(x,y)$ can be computed.  
At the node level, one often aggregates adjacent curvatures:
\begin{equation}
\kappa_x=\sum_{y\in \mathcal{N}(x)}\kappa(x,y)\quad\text{(or normalized average)}.
\end{equation}
Here $\kappa_x>0$ indicates local contraction, while $\kappa_x<0$ signals local expansion or bottlenecks.

\subsection{Consistency with the smooth case}
Finally, if $M$ is a smooth Riemannian manifold as in Section~\ref{sec:ricci_curv+on_riem_mfd} and $p_x$ is taken as the uniform measure on a small ball $m_{r,x}$, then \eqref{eq:ORC} recovers the expansion
\begin{equation}
\kappa(x,y)\;\sim\;\frac{\mathrm{Ric}(v)}{2(n+2)}\,r^2,\qquad r\to 0,
\end{equation}
as noted in Example 7 of~\citep{ollivier2009ricci}.  
Thus, after rescaling by $r^{-2}$, Ollivier curvature converges to classical Ricci curvature, confirming that the discrete notion is consistent with the smooth theory.

\section{Dirac-limit consistency of Sobolev--Ricci Curvature and Ollivier--Ricci Curvature}
\label{sec:appendix-dirac-consistency}

In this section we develop the framework needed to compare ST and
ORC on graphs.

We begin by introducing the graph-theoretic setting and notation, including
the root, shadow sets, and probability measures on the vertex set. 
We then define the ST $S_p$ via graph-based Sobolev
functions and present its closed-form expressions. 
On this basis, we introduce Sobolev curvature as a natural analogue of
ORC, both formulated through ratios of transport distances (See Section~\ref{sec:preliminaries_dirac_limit_of_SRC}). 

Finally, we establish auxiliary lemmas and propositions that show consistency in
the Dirac-limit, ensuring that both curvatures behave as expected when measures
concentrate at single vertices.
We consider two Dirac-limit regimes: (i) Gaussian localization as $\sigma\to0$, and
(ii) the lazy random-walk neighborhood measures $\mu_{x,\alpha}$ as $\alpha\to1$
(See Sections~\ref{sec:auxiliary_lemmas_dirac_limit} and~\ref{sec:main_propositions_dirac_limit}).

\subsection{Preliminaries}\label{sec:preliminaries_dirac_limit_of_SRC}

\paragraph{Setting and notation.}
Let $G=(V,E)$ be a connected undirected graph regarded as a 1D metric space
equipped with the shortest-path distance $d$.
Fix a \emph{root} $z_0\in V$ and assume the uniqueness of shortest paths to the root:
for every $y\in V$, there exists a \emph{unique} shortest path $[z_0,y]$ in $G$.
Let $\lambda$ be a nonnegative Borel measure on $G$ such that (i) $\lambda(\{x\})=0$ for every $x\in V$
(i.e.\ no atoms), and (ii) $\lambda(e)<\infty$ for every edge $e\in E$.
For $x\in V$, define
\begin{equation}
\Lambda(x):=\{\,y\in V:\ x\in [z_0,y]\,\}.
\end{equation}

For an edge $e\in E$, define the \emph{shadow} set
\begin{equation}
\gamma_e:=\{\,y\in V:\ e\subset [z_0,y]\,\}.    
\end{equation}

We denote by $\mathcal{P}(V)$ the set of all probability measures supported on the vertex set $V$.
For $\mu_n,\mu \in \mathcal{P}(V)$,
we write $\mu_n \Rightarrow \mu$ if
\begin{equation}
   \int_V f(x)\,\mu_n(dx) \;\longrightarrow\; \int_V f(x)\,\mu(dx)
   \qquad \text{for all bounded continuous functions } f:V\to\mathbb{R},    
\end{equation}
i.e., $\mu_n$ converges weakly to $\mu$.

\paragraph{Graph-Sobolev space and Sobolev transport.}
Let $1\le p\le\infty$ and $p'$ be its conjugate.
A continuous function $f:G\to\mathbb{R}$ belongs to the graph-based Sobolev space
$W^{1,p'}(G,\lambda)$ if there exists $h\in L^{p'}(G,\lambda)$ such that
\[
f(x)-f(z_0)=\int_{[z_0,x]} h(y)\,\lambda(dy)\quad \text{for all }x\in V,
\]
in which case we call $f'=h$ the \emph{graph derivative} of $f$.
The (order-$p$) ST distance between $\mu,\nu\in\mathcal{P}(V)$ is
\begin{equation}
S_p(\mu,\nu):=\sup\Big\{\ \int f\,d\mu-\int f\,d\nu\ :\ f\in W^{1,p'}(G,\lambda),\ \|f'\|_{L^{p'}(\lambda)}\le 1\ \Big\}.    
\end{equation}

It is known that $S_p$ is a metric on $\mathcal{P}(V)$ and admits the closed form
\begin{equation}\label{eq:Sp-closed}
S_p(\mu,\nu)^p=\int_G \big|\mu(\Lambda(x))-\nu(\Lambda(x))\big|^p\,\lambda(dx),
\end{equation}
and, because $\lambda$ is non-atomic, the discrete form
\begin{equation}\label{eq:Sp-discrete}
S_p(\mu,\nu)^p=\sum_{e\in E}\lambda(e)\,\big|\mu(\gamma_e)-\nu(\gamma_e)\big|^p.
\end{equation}

\paragraph{Curvatures.}
Given distinct $x,y \in V$ and measures $\mu,\nu \in \mathcal{P}(V)$ ``centered'' at $x,y$,
we define the Sobolev curvature by
\begin{equation}
\kappa_{\mathrm{S_p}}(\mu,\nu;x,y)
:=1-\frac{S_p(\mu,\nu)}{S_p(\delta_x,\delta_y)}
\quad (\text{and set }\kappa_{\mathrm{S_p}}(\mu,\nu;x,x):=0),
\end{equation}
and the Ollivier--Ricci curvature by
\begin{equation}
\kappa_{\mathrm{ORC}}(\mu,\nu;x,y)
:=1-\frac{W_1(\mu,\nu)}{d(x,y)}
\quad (\text{and set }\kappa_{\mathrm{ORC}}(\mu,\nu;x,x):=0).
\end{equation}

\subsection{Auxiliary lemmas}\label{sec:auxiliary_lemmas_dirac_limit}

As an illustrative example, recall the localized measures 
$\mu_x^{(\sigma)} \in \mathcal{P}(V)$ defined in Section~\ref{sec:sobolev_curvature}:
\begin{equation}
\mu_x^{(\sigma)}(v) \;\propto\;
   \exp\!\Big(-\tfrac{\|x-v\|_p^2}{\sigma^2}\Big)\,
   \mathbf{1}_{\{v \in \mathcal{N}(x)\}},    
\end{equation}

which converge weakly to the Dirac measure $\delta_x$ as $\sigma \to 0$.
Lemma~\ref{lem:gamma-stability} below shows that any such family 
$(\mu^{(\sigma)})_{\sigma>0} \subseteq \mathcal{P}(V)$ converging weakly to 
$\delta_x$ satisfies the desired stability property.

\begin{lemma}[Stability of $\gamma_e$ under weak convergence]\label{lem:gamma-stability}
Fix $x\in V$ and an edge $e\in E$.
Let $(\mu^{(\sigma)})_{\sigma>0}\subset\mathcal{P}(V)$ satisfy $\mu^{(\sigma)}\Rightarrow \delta_x$ as $\sigma\downarrow 0$.
Then $\mu^{(\sigma)}(\gamma_e)\to \delta_x(\gamma_e)$ as $\sigma\downarrow 0$.
\end{lemma}

\begin{proof}
First, $\gamma_e$ is closed: if $y_n\to y$ and $e\subset [z_0,y_n]$ for all $n$,
uniqueness of shortest paths to $z_0$ implies $[z_0,y_n]\to [z_0,y]$ as sets,
hence $e\subset [z_0,y]$.

We distinguish the three cases depending on the relative position of $x$ with respect to $\gamma_e$:

If $x\notin \gamma_e$, Portmanteau theorem~\citep{billingsley2013convergence} gives $\limsup_{\sigma}\mu^{(\sigma)}(\gamma_e)\le \delta_x(\gamma_e)=0$.

If $x\in \mathrm{int}(\gamma_e)$, then for some open $U\subset\gamma_e$ with $x\in U$,
Portmanteau yields $\liminf_{\sigma}\mu^{(\sigma)}(\gamma_e)\ge \liminf_{\sigma}\mu^{(\sigma)}(U)\ge 1$,
hence the limit is $1=\delta_x(\gamma_e)$.

If $x$ lies on $\partial\gamma_e$, then necessarily $\delta_x(\gamma_e)=0$ and
again $\limsup_{\sigma}\mu^{(\sigma)}(\gamma_e)\le 0$. Thus in all cases $\mu^{(\sigma)}(\gamma_e)\to \delta_x(\gamma_e)$.
\end{proof}

\begin{lemma}[Vanishing of $S_p$ to a Dirac]\label{lem:Sp-to-dirac}
Under the same assumptions as above, for any $1\le p<\infty$,
\begin{equation}
S_p\!\big(\mu^{(\sigma)},\delta_x\big)\ \xrightarrow[\sigma\downarrow 0]{}\ 0.    
\end{equation}
\end{lemma}

\begin{proof}
Apply \eqref{eq:Sp-discrete} with $\nu=\delta_x$:
\begin{equation}
S_p\!\big(\mu^{(\sigma)},\delta_x\big)^p
=\sum_{e\in E}\lambda(e)\,\big|\mu^{(\sigma)}(\gamma_e)-\delta_x(\gamma_e)\big|^p.    
\end{equation}

By Lemma~\ref{lem:gamma-stability}, each summand converges to $0$.
Since $|\mu^{(\sigma)}(\gamma_e)-\delta_x(\gamma_e)|\le 1$ and $\sum_{e\in E}\lambda(e)<\infty$,
the dominated convergence theorem gives $S_p(\mu^{(\sigma)},\delta_x)\to 0$.
\end{proof}

\begin{lemma}[Two basic facts for $W_1$]\label{lem:W1-basics}
For any metric space $(G,d)$ and any $x,y\in V$:
\begin{enumerate}
\item $W_1(\delta_x,\delta_y)=d(x,y)$.
\item If $\mu^{(\sigma)}\Rightarrow \delta_x$ and either $\mathrm{diam}(G)<\infty$ or $\int d(z,x)\,\mu^{(\sigma)}(dz)\to 0$,
then $W_1(\mu^{(\sigma)},\delta_x)\to 0$.
\end{enumerate}
\end{lemma}

\begin{proof}
(1) By the Kantorovich--Rubinstein duality,
$W_1(\delta_x,\delta_y)=\sup_{\mathrm{Lip}(f)\le 1}\{f(x)-f(y)\}=d(x,y)$
(the upper bound is immediate; equality is attained e.g.\ by $f=-d(\cdot,x)$).
(2) If $\mathrm{diam}(G)<\infty$, then the class of 1-Lipschitz functions is uniformly bounded and equicontinuous,
hence weak convergence implies convergence of the dual objective to $0$; thus $W_1(\mu^{(\sigma)},\delta_x)\to 0$.
Alternatively, for any $\sigma$, consider the coupling 
$\pi^{(\sigma)}=(\mathrm{Id},x)_{\#}\mu^{(\sigma)}$, i.e., the pushforward of $\mu^{(\sigma)}$
under the map $z\mapsto (z,x)$. This is the trivial coupling sending all mass of $\mu^{(\sigma)}$ 
to the point $x$, and it yields
\begin{equation}
W_1(\mu^{(\sigma)},\delta_x)\le \int d(z,x)\,\mu^{(\sigma)}(dz),    
\end{equation}
which vanishes by assumption.
\end{proof}

\subsection{Main propositions}\label{sec:main_propositions_dirac_limit}
We present Dirac-limit flatness results for two concentration mechanisms:
Gaussian localization as $\sigma\downarrow0$ and the lazy random-walk family as $\alpha\to1$ (used in SRF).

\paragraph{Dirac-limit flatness via Gaussian / heat-kernel-like.}
We first treat the canonical Dirac concentration $\mu_x^{(\sigma)}\Rightarrow\delta_x$ as $\sigma\downarrow0$,
which covers Gaussian/heat-kernel localizations.

\begin{proposition}[Dirac-limit consistency of $S_p$]\label{prop:SRC-dirac}
Let $x,y\in V$ and $1\le p<\infty$.
Suppose $\mu_x^{(\sigma)},\mu_y^{(\sigma)}\in\mathcal{P}(V)$ satisfy
$\mu_x^{(\sigma)}\Rightarrow \delta_x$ and $\mu_y^{(\sigma)}\Rightarrow \delta_y$ as $\sigma\downarrow 0$.
Then
\begin{equation}
S_p\!\big(\mu_x^{(\sigma)},\mu_y^{(\sigma)}\big)\ \longrightarrow\ S_p(\delta_x,\delta_y),
\qquad\text{hence}\qquad
\kappa_{\mathrm{S_p}}\!\big(\mu_x^{(\sigma)},\mu_y^{(\sigma)};x,y\big)\ \longrightarrow\ 0.
\end{equation}
\end{proposition}

\begin{proof}
By the triangle inequality for the metric $S_p$,
\begin{equation}
\big|S_p(\mu_x^{(\sigma)},\mu_y^{(\sigma)})-S_p(\delta_x,\delta_y)\big|
\le S_p(\mu_x^{(\sigma)},\delta_x)+S_p(\mu_y^{(\sigma)},\delta_y).    
\end{equation}

Both terms on the right vanish by Lemma~\ref{lem:Sp-to-dirac}.
If $x\neq y$ then $S_p(\delta_x,\delta_y)>0$ and the curvature converges to $0$ by definition;
if $x=y$ the claim is trivial by our convention.
\end{proof}

\begin{proposition}[Dirac-limit consistency of ORC]\label{prop:ORC-dirac}
Let $x,y\in V$.
Suppose $\mu_x^{(\sigma)}\Rightarrow \delta_x$ and $\mu_y^{(\sigma)}\Rightarrow \delta_y$ as $\sigma\downarrow 0$,
and either $\mathrm{diam}(G)<\infty$ or
$\int d(z,x)\,\mu_x^{(\sigma)}(dz)\to 0$ and $\int d(w,y)\,\mu_y^{(\sigma)}(dw)\to 0$.
Then
\begin{equation}
W_1\!\big(\mu_x^{(\sigma)},\mu_y^{(\sigma)}\big)\ \longrightarrow\ d(x,y),
\qquad\text{hence}\qquad
\kappa_{\mathrm{ORC}}\!\big(\mu_x^{(\sigma)},\mu_y^{(\sigma)};x,y\big)\ \longrightarrow\ 0.    
\end{equation}
\end{proposition}

\begin{proof}
By the triangle inequality for $W_1$ and Lemma~\ref{lem:W1-basics},
\begin{equation}
\big|W_1(\mu_x^{(\sigma)},\mu_y^{(\sigma)})-W_1(\delta_x,\delta_y)\big|
\le W_1(\mu_x^{(\sigma)},\delta_x)+W_1(\mu_y^{(\sigma)},\delta_y)\ \longrightarrow\ 0.
\end{equation}

Since $W_1(\delta_x,\delta_y)=d(x,y)$ by Lemma~\ref{lem:W1-basics}(1),
we obtain the desired convergence of $W_1$ and hence of the curvature.
\end{proof}

\paragraph{Dirac-limit flatness via lazy random walk.}
For the flow-type neighborhood measure with parameter $\alpha\in(0,1]$,
define
\begin{equation}\label{eq:mu_alpha}
\mu_{x,\alpha} \;:=\; \alpha\,\delta_x \;+\; (1-\alpha)\,\nu_x,
\qquad
\nu_x \;:=\; \frac{1}{|N(x)|}\sum_{v\in N(x)}\delta_v .
\end{equation}
Then $\mu_{x,\alpha}\Rightarrow\delta_x$ as $\alpha\to1$ (on the finite set $V$, this is equivalent to
pointwise convergence of masses).

Recall that ST on a tree satisfying the unique-path-to-root property is
\[
S_p(\mu,\nu)^p \;=\; \sum_{e\in E}\lambda(e)\,\big|\mu(\mathbf{1}_{T_e})-\nu(\mathbf{1}_{T_e})\big|^p .
\]
Fix an edge $e\in E$ and write $m_{x,\alpha}(e):=\mu_{x,\alpha}(\mathbf{1}_{T_e})$.
By \eqref{eq:mu_alpha} and linearity,
\[
m_{x,\alpha}(e)
= \alpha\,\delta_x(\mathbf{1}_{T_e}) + (1-\alpha)\,\nu_x(\mathbf{1}_{T_e})
\xrightarrow[\alpha\to1]{}
\delta_x(\mathbf{1}_{T_e}).
\]
Hence for any $(x,y)$,
\[
\big|m_{x,\alpha}(e)-m_{y,\alpha}(e)\big|^p
\;\xrightarrow[\alpha\to1]{}\;
\big|\delta_x(\mathbf{1}_{T_e})-\delta_y(\mathbf{1}_{T_e})\big|^p.
\]
Since $0\le m_{x,\alpha}(e)\le1$, the above differences are uniformly bounded,
and the sum over $e\in E$ is finite; therefore we may pass the limit inside the sum to obtain
\[
\lim_{\alpha\to1} S_p(\mu_{x,\alpha},\mu_{y,\alpha})^p
=
\sum_{e\in E}\lambda(e)\,\big|\delta_x(\mathbf{1}_{T_e})-\delta_y(\mathbf{1}_{T_e})\big|^p
=
S_p(\delta_x,\delta_y)^p
= D_p(x,y)^p .
\]
Taking $p$-th roots yields $S_p(\mu_{x,\alpha},\mu_{y,\alpha})\to D_p(x,y)$, and thus
\[
\kappa_{S_p}^{(\alpha)}(x,y)
= 1 - \frac{S_p(\mu_{x,\alpha},\mu_{y,\alpha})}{D_p(x,y)}
\xrightarrow[\alpha\to1]{} 0,
\]
showing that SRC is flat in the Dirac limit.

\section{Complexity Analysis for Ricci Flow Curvatures}
\label{sec:complexity_appendix}

This appendix provides a complexity analysis of the Ricci flow variants studied in this paper.
In particular, we compare the computational cost of SRC with that of standard ORC-based Ricci flow,
clarifying the dependence on graph size and tree construction.

\subsection{Notation and Setup}
\label{subsec:complexity_notation}

Let $G=(V,E)$ be a graph with $n = |V|$ nodes and $m = |E|$ edges.
Throughout, we assume that localized neighborhood measures $\mu_x$ are supported on at most $k+1$ nodes
(i.e., the node itself and its $k$ neighbors), hence the effective support size is $O(k)$.
We write $T_{\mathrm{sk}}$ for the number of Sinkhorn iterations used in the entropic OT solver.
We analyze the cost of \emph{one Ricci-flow iteration}, i.e., one full pass of (i) curvature computation on edges and (ii) metric update.

In practice, the overall runtime is also affected by preprocessing steps such as all-pairs shortest paths or nearest-neighbor search.
These steps depend on the implementation and experimental setting; here we focus on the dominant per-iteration terms
that explain the relative cost among ORC and SRC variants.

\subsection{ORC Complexity (OTD/Sinkhorn Mix)}
GraphRicciCurvature computes ORC by solving an OT problem on each edge.
Here $k$ denotes the support size of local measures $\mu_x$, which in our graph setting is induced by one-hop neighborhoods, i.e., $|\mathrm{supp}(\mu_x)|=\deg(x)+1$ (Sec.~\ref{subsec:complexity_notation}).
In our experiments, we use the default option \texttt{OTDSinkhornMix},
which applies exact OT (OTD) for small neighborhoods and switches to Sinkhorn for large ones.
Thus, the per-edge cost is
$\tilde{O}(k^3)$ in the OTD regime and $\tilde{O}(T_{sk} k^2)$ in the Sinkhorn regime,
leading to an overall per-iteration cost of
$\tilde{O}(|E| \cdot \min\{k^3,\, T_{sk} k^2\})$
(up to preprocessing).

\paragraph{Cost per edge.}
With an entropic regularization, the Sinkhorn algorithm iteratively updates dual potentials or a transport scaling.
Each Sinkhorn iteration typically involves operations on a dense $k\times k$ kernel matrix,
leading to $O(k^2)$ arithmetic per iteration.
Therefore, the OT cost per edge scales as
\begin{equation}
  \tilde{O}(T_{sk} k^2),
\end{equation}
where $\tilde{O}(\cdot)$ hides logarithmic factors and small constant overheads.
(Some implementations include additional steps such as small-scale exact EMD for tiny supports or mixing strategies,
which do not change the dominant scaling.)

\paragraph{Cost per flow iteration.}
Computing ORC on all edges yields the per-iteration complexity
\begin{equation}
  \tilde{O}(T_{\mathrm{sk}}\, m\, k^2).
  \label{eq:orc_cost}
\end{equation}
Implementation-dependent preprocessing may additionally appear (e.g., shortest-path computations for initial edge lengths),
but the dominant term in ORC-based Ricci flow is the repeated OT solve per edge,
which results in large constant factors in practice.

\subsection{SRC Complexity: Tree Sobolev Transport}
\label{subsec:complexity_src}

SRC replaces the Wasserstein OT distance with ST on a tree-metric.
Given a rooted tree representation, the ST distance between two measures can be computed
by aggregating mass differences over tree edges, avoiding iterative OT solvers.

\paragraph{Tree representation and flow vectors.}
Consider a rooted tree on $n$ nodes.
Let the set of tree edges be indexed as $\{e_1,\dots,e_{n-1}\}$, with associated edge lengths $\{\lambda(e)\}$.
A standard property of tree transport is that for $p=1$,
the transport cost between $\mu_u$ and $\mu_v$ can be written as
\begin{equation}
  S_1(\mu_u,\mu_v)
  \;=\;
  \sum_{e \in \text{tree}} \lambda(e)\,\big|F_u(e)-F_v(e)\big|,
  \label{eq:tree_transport_flowform}
\end{equation}
where $F_u(e)$ denotes the total mass of $\mu_u$ whose root-path crosses edge $e$.
Hence the core computation reduces to constructing these aggregated values along the tree.

\paragraph{Cost to build aggregated representations.}
For each node $x$, the measure $\mu_x$ has support size $O(k)$.
To accumulate $F_x(e)$, one may push each support mass upward to the root along parent pointers.
If the average root-to-support depth is $h$ (depending on the tree construction),
then each support point requires $O(h)$ steps,
and the total aggregation cost over all nodes is
\begin{equation}
  O(n \cdot k \cdot h).
  \label{eq:src_aggregation_cost}
\end{equation}

\paragraph{Cost to evaluate curvature on edges.}
Once the aggregated tree representation is available,
computing $S_1(\mu_u,\mu_v)$ for a given edge $(u,v)$ requires evaluating the sum in~\eqref{eq:tree_transport_flowform}.
In our reference implementation, we explicitly store per-node flow vectors over all tree edges of the extracted tree $T_r$.
Since $|E(T_r)| = n-1$, evaluating $S_1(\mu_u,\mu_v)$ costs $O(n)$ time per graph edge.
Thus, curvature evaluation over all graph edges scales as
\begin{equation}
  O(mn)
  \quad \text{(dense evaluation)}.
  \label{eq:src_eval_cost_dense}
\end{equation}

The $O(mn)$ cost corresponds to dense evaluation; exploiting sparsity of neighborhood supports on the tree could reduce the practical cost.
Combining~\eqref{eq:src_aggregation_cost} and~\eqref{eq:src_eval_cost_dense}, we obtain the SRC per-iteration cost
\begin{equation}
  \tilde{O}(\texttt{tree-construction} \;+\; n k h \;+\; mn).
  \label{eq:src_generic_cost}
\end{equation}
Although SRC admits a closed-form evaluation on trees, in practice optimized Sinkhorn solvers can be faster in certain regimes due to highly optimized dense linear-algebra kernels (e.g., GPU/BLAS), despite the iterative nature of Sinkhorn.

\subsection{Tree Construction: SRC(SPT) vs.\ SRC(MST)}
\label{subsec:complexity_tree_construction}

SRC requires a tree structure, constructed at each Ricci-flow iteration using the current edge lengths.

\paragraph{SRC(SPT).}
SRC(SPT) constructs a shortest-path tree (SPT) rooted at a fixed node.
On a sparse graph, Dijkstra's algorithm yields
\begin{equation}
  \tilde{O}(m \log n)
  \label{eq:spt_cost}
\end{equation}
per iteration.

\paragraph{SRC(MST).}
SRC(MST) constructs a minimum spanning tree (MST).
Using standard algorithms (e.g., Kruskal), MST construction costs
\begin{equation}
  \tilde{O}(m \log n),
  \label{eq:mst_cost}
\end{equation}
followed by a linear-time rooting (BFS/DFS).
Thus, both trees have similar asymptotic construction cost.

\paragraph{Why SRC(SPT) is faster in practice.}
Although~\eqref{eq:spt_cost} and~\eqref{eq:mst_cost} share the same asymptotic order,
the subsequent transport aggregation depends on the induced average tree depth $h$.
SPT tends to preserve locality with respect to the current graph metric,
which typically yields shorter root-to-support paths than MST.
Consequently,
\begin{equation}
  h_{\mathrm{SPT}} < h_{\mathrm{MST}}
  \quad \Rightarrow \quad
  nk h_{\mathrm{SPT}} < nk h_{\mathrm{MST}},
\end{equation}
which explains the empirical ordering $\text{SRC(SPT)} < \text{SRC(MST)}$
observed in our runtime measurements.

\subsection{Putting Everything Together}
\label{subsec:complexity_summary}

Substituting tree construction costs into~\eqref{eq:src_generic_cost} yields:
\begin{align}
\mathrm{SRC(SPT)} &: \tilde{O}\!\left(m\log n \;+\; nk h_{\mathrm{SPT}} \;+\; mn\right), \\
\mathrm{SRC(MST)} &: \tilde{O}\!\left(m\log n \;+\; nk h_{\mathrm{MST}} \;+\; mn\right).
\end{align}
In comparison, ORC with Sinkhorn scales as~\eqref{eq:orc_cost}:
\begin{equation}
\mathrm{ORC}: \tilde{O}\!\left(T m k^2\right).
\end{equation}
Since $T$ is typically tens to hundreds in practice,
and since ORC requires repeated dense $k\times k$ updates per edge,
we expect ORC to be substantially more expensive than SRC-based flows.
Moreover, the difference $h_{\mathrm{SPT}} < h_{\mathrm{MST}}$ explains why SRC(SPT) is faster than SRC(MST) in practice,
even though both share the same asymptotic tree-construction cost.

\subsection{Concrete Scaling on SBM ($n=500$)}
\label{subsec:complexity_sbm_numbers}

To connect the above asymptotic bounds with our SBM runtime results,
we provide a back-of-the-envelope scaling estimate under the SBM setting in Table~\ref{tab:dataset_settings_synth} of Appendix~\ref{app:exp_settings}:
two equal-size communities ($250$+$250$), $P_{\mathrm{intra}}=0.15$, and
$P_{\mathrm{inter}}=\rho P_{\mathrm{intra}}$ with $\rho\in\{0.1,\dots,0.9\}$.

\paragraph{Expected degree and edge count.}
For a node, the expected degree is
\begin{equation}
\mathbb{E}[\deg] \;=\; (250-1)p_{\mathrm{in}} + 250 p_{\mathrm{out}}
\;=\; 37.35 + 37.5\rho,
\end{equation}
which ranges from $\approx 41$ ($\rho=0.1$) to $\approx 71$ ($\rho=0.9$).
The expected number of edges is
\begin{equation}
\mathbb{E}[m] \;=\; 2\binom{250}{2}p_{\mathrm{in}} + 250^2 p_{\mathrm{out}}
\;=\; 9337.5 + 9375\rho,
\end{equation}
ranging from $\approx 1.0\times 10^{4}$ to $\approx 1.8\times 10^{4}$.

\paragraph{Implication for ORC vs.\ SRC.}
In this SBM regime, the number of edges is on the order of $10^4$ (i.e., $m=\Theta(10^4)$) and the effective neighborhood support size is $O(k)$ with
$k$ on the order of the average degree.
Thus, the ORC per-iteration cost $\tilde{O}(T_{sk} m k^2)$ grows quadratically in $k$.
Increasing $\rho$ from $0.1$ to $0.9$ increases the expected degree from $\approx 41$ to $\approx 71$,
which amplifies the $k^2$ factor by roughly $(71/41)^2 \approx 3$, even before accounting for
the Sinkhorn iteration count $T_{sk}$.

In contrast, SRC avoids per-edge OT. In our reference implementation, the per-iteration cost is
$\tilde{O}(m\log n + nk\,h + mn)$, where the $nk\,h$ term accounts for aggregating measures over supports of size $k$ on the extracted tree, and the $mn$ term comes from dense curvature evaluation over $|E(T_r)|=n-1$ tree edges for each graph edge.
The factor $h$ reflects the depth of the induced tree used by ST; in practice, SPTs rooted at $z_0$ tend to have smaller depth than MSTs on these graphs.

Consequently, these considerations suggest that SRC(SPT) can be faster than SRC(MST) due to smaller $h$,
and that ORC can be substantially slower in this setting, consistent with the empirical runtimes
reported in Fig.~\ref{fig:sbm_lfr_real_runtime}(d).

\section{Experimental settings for community detection}\label{app:exp_settings}

This appendix documents the experimental protocol for reproducibility.
Table~\ref{tab:common_hparams} lists hyperparameters shared across all experiments
(e.g., flow iterations, SRC parameters, and Louvain settings).
For ORC-based baselines, we use the default settings of the \texttt{GraphRicciCurvature} and \texttt{iGraph} implementations, without additional tuning.
We run the flow for at most $T_{\mathrm{flow}}$ iterations and apply early stopping when the curvature change falls below a threshold $\varepsilon$, where
\begin{equation}
\Delta\kappa^{(t)} := \max_{e\in E}\left|\kappa_{e}^{(t)} - \kappa_{e}^{(t-1)}\right|,
\end{equation}

We set the maximum number of SRC flow iterations $T_{\mathrm{flow}}$ depending on the purpose of each experiment.
For the main SBM/LFR evaluations and the associated runtime measurements (Fig.~\ref{fig:sbm_lfr_real_runtime}(a),(b),(d)), we use
$T_{\mathrm{flow}}=50$ to match the standard setting and to enable a fair comparison under the same iteration budget.
For experiments on real-world networks, we use $T_{\mathrm{flow}}=20$ to prioritize stability and reproducibility,
as we observed that increasing the number of iterations can lead to unstable behavior on some small graphs.
For the remaining auxiliary analyses, $T_{\mathrm{flow}}=20$ is used since our goal is to capture qualitative trends
rather than to fully optimize the final converged solution.
We employ an early-stopping rule based on the change in curvature; however, under the above settings it was rarely triggered.
For ORC-based Ricci flow (\texttt{GraphRicciCurvature}), we follow the default configuration and use a maximum of 10 iterations.

Table~\ref{tab:dataset_settings_synth} and Table~\ref{tab:dataset_stats_real} summarize dataset-specific configurations.
Section~\ref{sec:profiles_of_datasets} provides dataset sources and licenses, and Section~\ref{sec:machine_spec} provides machine specification.

\begin{table}[t]
\centering
\caption{Common hyperparameters used across community detection experiments for SRF.}
\small
\begin{tabular}{l l}
\toprule
\textbf{Item} & \textbf{Value} \\
\midrule
Flow iterations $T_{\mathrm{flow}}$ & $20$ , early stop if $\Delta\kappa^{(t)} < \varepsilon$\\
Lazy RW parameter $\alpha$ & $0.5$\\
ST exponent $p$ (SRC) & $1$\\
Tree construction & MST / SPT \\
Root selection (SPT) &  $0$\\
Length-to-affinity $\beta$ & $1.0$\\
\bottomrule
\end{tabular}
\label{tab:common_hparams}
\end{table}

\begin{table}[t]
\centering
\small
\renewcommand{\arraystretch}{1.0}

\begin{minipage}[t]{0.57\linewidth}
\centering
\caption{Dataset-specific settings for synthetic graphs (SBM/LFR).}
\begin{tabular}{l p{0.78\linewidth}}
\toprule
\textbf{Dataset} & \textbf{Parameters} \\
\midrule
SBM &
$n=500$, $K=2$ (equal-size);\par
$P_{\mathrm{intra}}=0.15$, $P_{\mathrm{inter}}=\rho P_{\mathrm{intra}}$, $\rho\in\{0.1,0.2,\dots,0.9\}$;\par
$\mathbb{E}[\deg]=37.35+37.5\rho$, range $[41,71]$;\par
$\mathbb{E}[m]=9337.5+9375\rho$, range $[1.0,1.8]\times10^{4}$;\par
instances $=10$.\\
\midrule
LFR &
$n=500$;\par
$\mu\in\{0.1,0.2,\dots,0.9\}$;\par
$\bar{k}=20$, $k_{\max}=50$;\par
$\tau_1=3.0$, $\tau_2=1.5$;\par
sizes $\in[20,100]$;\par
$K$ from ground truth;\par
instances $=10$.\\
\bottomrule
\end{tabular}
\label{tab:dataset_settings_synth}
\end{minipage}
\hfill
\begin{minipage}[t]{0.40\linewidth}
\centering
\caption{Statistics of real-network benchmarks.}
\normalsize
\begin{tabular}{lrrr}
\toprule
\textbf{Dataset} & $n$ & $m$ & $K$ \\
\midrule
karate & 34 & 78 & 2 \\
football & 115 & 613 & 12 \\
polbooks & 105 & 441 & 3 \\
polblogs & 1222 & 16714 & 2 \\
email-eu-core & 986 & 16064 & 42 \\
\bottomrule
\end{tabular}
\label{tab:dataset_stats_real}
\end{minipage}

\end{table}

\subsection{Profiles of datasets}\label{sec:profiles_of_datasets}

We summarize below the datasets used in our experiments, along with their 
sources and licenses. For synthetic datasets, we specify the generation 
procedure or implementation used.

\begin{itemize}
    \item \textbf{Stochastic Block Model (SBM).}
    We generate graphs from the stochastic block model with a fixed number of nodes
    and planted community structure.
    Each graph is constructed by specifying intra-community and inter-community connection
    probabilities, denoted by $P_{\mathrm{intra}}$ and $P_{\mathrm{inter}}$, respectively.
    We report results across a range of separation regimes controlled by the ratio
    $P_{\mathrm{inter}} / P_{\mathrm{intra}}$.
    Unless otherwise specified, we use $n=500$ nodes and a fixed number of communities,
    and average results over multiple random seeds.

    \item \textbf{Lancichinetti--Fortunato--Radicchi (LFR).}
    We also evaluate on LFR benchmark graphs, which exhibit heterogeneous degree
    distributions and community sizes.
    The primary parameter is the mixing coefficient $\mu\in[0,1]$, which controls the fraction
    of edges that connect to nodes outside the ground-truth community.
    Smaller $\mu$ corresponds to clearer community structure.
    Unless otherwise specified, we generate graphs with $n=500$ nodes and vary $\mu$
    over a predefined grid, averaging results over multiple random seeds.

    \item \textbf{Real networks.}
    Finally, we evaluate on commonly used real-world community detection benchmarks:
    \texttt{karate} (Zachary's Karate Club), \texttt{football} (American college football),
    \texttt{polbooks} (Political Books), \texttt{polblogs} (Political Blogs),
    and \texttt{email-eu-core} (Email-Eu-core network).
    For each dataset, we use the standard ground-truth community labels provided
    in the original benchmark sources and report clustering quality using ARI.
\end{itemize}

\subsection{Machine specification}\label{sec:machine_spec}
All experiments were run on a server running Ubuntu 20.04 with 
two Intel(R) Xeon(R) Gold 6354 CPUs (3.00GHz, 36 cores each) and 252GB RAM.

\section{Ablation study on community detection}\label{app:ablation_study_on_comm_detection}

This appendix provides additional experimental details and ablation results for the Ricci-flow-based community detection pipeline.

We first summarize the downstream community detection procedure used throughout the paper (Louvain modularity maximization and the length-to-affinity conversion) in Section~\ref{app:louvain}, and then report ablations that isolate the effect of modeling choices in SRC, including (i) the transport exponent $p$ and the choice of downstream clustering rule in Section~\ref{sec:ablation_p_downstream_lfr}, (ii) the neighborhood-measure parameter $\alpha$ in the lazy random walk in Section~\ref{app:ablation-alpha-lfr},~\ref{app:sensitivity_to_alpha_small_nw}, (iii) the temperature parameter of $\beta$ in Section~\ref{sec:ablation_beta},  (iv) curvature distributions and statistics in Section~\ref{app:curvature_histograms_lfr}, and
(v) robustness to tree construction (random spanning trees) in Section~\ref{app:tree_robustness}.

\subsection{Louvain community detection}\label{app:louvain}

We use the Louvain method as the downstream community-detection routine after running Ricci-flow reweighting.
Given an undirected weighted graph $(V,E,\tilde w)$ with nonnegative affinity weights $\tilde w_{uv}$,
Louvain seeks a partition $\mathcal{C}=\{C_1,\dots,C_K\}$ that (approximately) maximizes the modularity
\begin{equation}\label{eq:modularity}
Q(\mathcal{C})
=\frac{1}{2m}\sum_{x,y\in V}\Big(\tilde w_{xy}-\frac{k_x k_y}{2m}\Big)\,\mathbf{1}\{c_x=c_y\},
\end{equation}
where $k_x=\sum_{y}\tilde w_{xy}$ is the weighted degree of $u$ and $m=\frac{1}{2}\sum_{x,y}\tilde w_{xy}$.
The Louvain algorithm performs greedy local node moves that increase $Q$, followed by graph coarsening
that aggregates each community into a super-node, and repeats these two phases until convergence.
In our experiments, Ricci flow produces edge \emph{lengths} $w_{uv}$; since modularity-based methods assume affinities,
we convert lengths to similarities via $\tilde w_{xy}=\exp(-\beta w_{xy})$ with fixed $\beta>0$ before applying Louvain.

\subsection{Ablation Study: Effect of $p$ and Downstream Clustering (LFR)}
\label{sec:ablation_p_downstream_lfr}

Our main experiments compare curvature-driven Ricci flows using community detection as a downstream task.
Among our SRC variants, SRC(SPT) is consistently the most computationally efficient,
while delivering competitive performance that is close to the ORC-based flow (see Section~\ref{sec:community_detection_with_ricciflow}).
For this reason, we focus the following ablation study on SRC(SPT) in order to isolate
the effect of modeling choices \emph{within} SRC without introducing additional computational confounders.

\paragraph{Setup.}
We conduct the ablation on the \textbf{LFR benchmark}, which provides a controlled knob for community recovery difficulty
via the mixing parameter $\mu$.
For each $\mu$, we run SRC(SPT)-based SRF for a fixed number of iterations,
and evaluate the resulting graph using two downstream clustering procedures:
\begin{itemize}
    \item \textbf{Louvain:} modularity-based community detection on similarity weights obtained from the learned edge lengths,
    \item \textbf{Length-cut:} direct clustering based on thresholding learned edge lengths as dissimilarities.
\end{itemize}
We repeat each configuration across multiple random graph instances and report both the mean and standard deviation of ARI.

\paragraph{Varying the SRC exponent $p$.}
SRC is defined using a ST cost $S_p(\mu_u,\mu_v)$, where $p$ controls how transport mass differences
are aggregated along the tree.
To assess sensitivity to this choice, we test $p\in\{1, 1.5, 2\}$.
Figure~\ref{fig:result_lfr_src_p_ablation_mu_series_method_p} summarizes the results as a function of $\mu$.

\paragraph{Results: Louvain is robust while Length-cut is unstable.}
The left panels of Fig.~\ref{fig:result_lfr_src_p_ablation_mu_series_method_p} show that
\textbf{Louvain-based evaluation yields stable performance} across different values of the SRC exponent $p$.
In the easy to moderately difficult regimes (e.g., $\mu \lesssim 0.3$), the ARI curves for
$p \in \{1, 1.5, 2\}$ largely overlap, indicating only a weak dependence on $p$.
Around the transition region ($\mu \approx 0.4$--$0.5$), performance drops sharply and the standard deviations increase,
reflecting the onset of a more challenging community recovery regime.
Beyond this region, performance degrades smoothly with no abrupt sensitivity to the choice of $p$.

In contrast, the right panels demonstrate that \textbf{length-cut evaluation is substantially more brittle}.
The achieved ARI is already low compared to Louvain-based evaluation, and some variability appears even at very low
mixing rates (e.g., $\mu=0.1$).
As $\mu$ increases, length-cut clustering quickly degenerates to near-zero ARI; once recovery fails, the results show
little dependence on the choice of $p$.
This behavior highlights that direct length-cut clustering can be unstable:
small fluctuations in the learned edge lengths, or in the threshold choice, can induce large changes in the resulting partitions.

\begin{figure}
    \centering
    \includegraphics[width=0.9\linewidth]{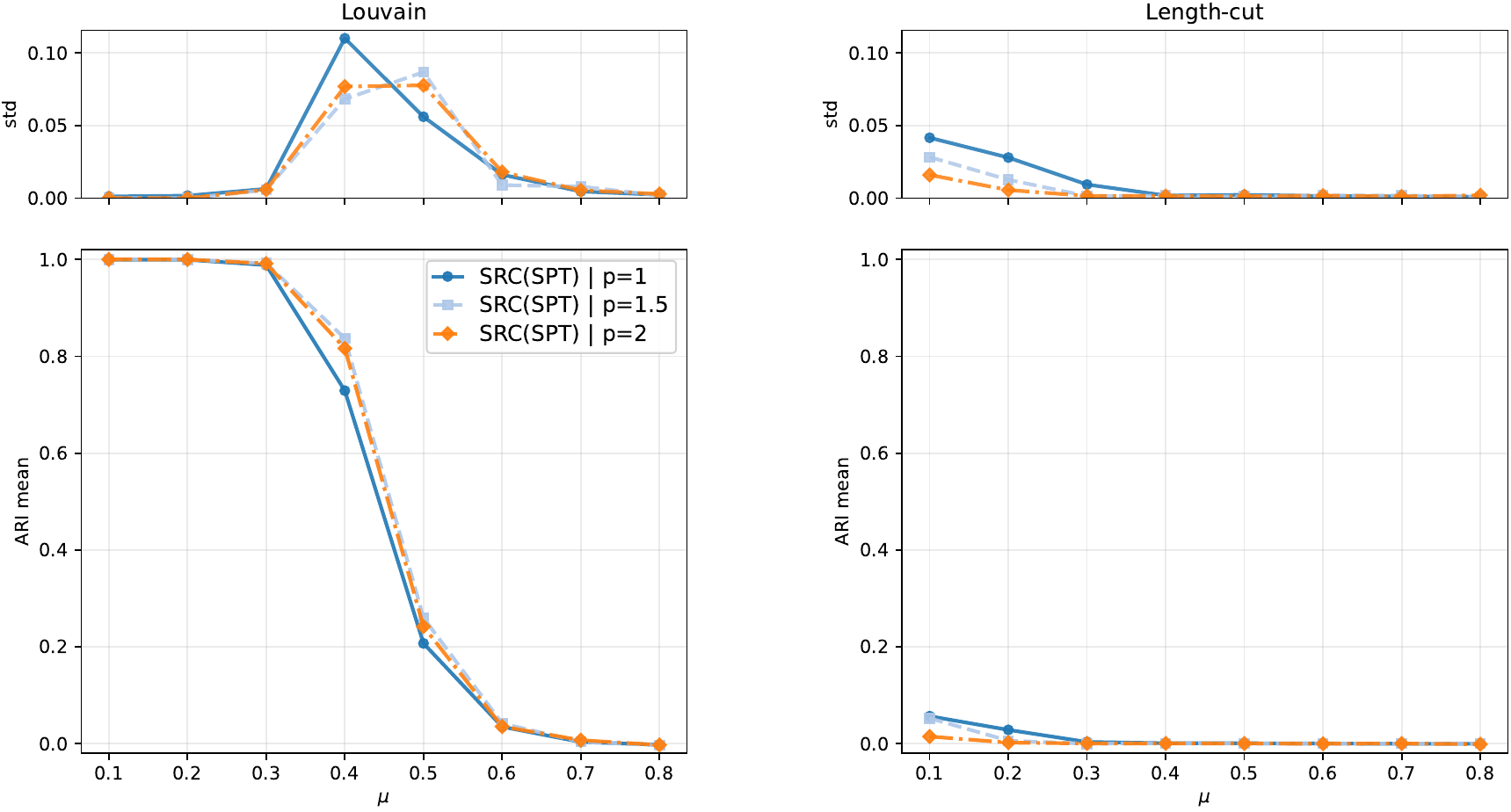}
    \caption{SRC(SPT) ablation on LFR: effect of the SRC exponent $p$ and downstream clustering.
We evaluate $p\in\{1,1.5,2\}$ using two downstream procedures:
Louvain on similarity weights (left) and threshold-based length-cut clustering (right).
Top panels show standard deviation of ARI across trials, and bottom panels show the mean ARI.
Louvain yields stable performance with little dependence on $p$,
whereas length-cut is unstable and substantially degrades ARI.}
    \label{fig:result_lfr_src_p_ablation_mu_series_method_p}
\end{figure}

\subsection{Ablation on the lazy random walk parameter $\alpha$ (LFR)}
\label{app:ablation-alpha-lfr}

In our framework, the curvature is computed from a neighborhood measure $\mu_x$ defined on the graph.
When using a lazy random walk in~\eqref{eq:mu_flow}, the measure is parameterized by $\alpha\in[0,1)$, which determines the self-loop (stay) probability.
Therefore, $\alpha$ directly controls the locality of the neighborhood distribution: larger $\alpha$ yields a more concentrated (near-Dirac) measure, while smaller $\alpha$ yields a more diffusive neighborhood.
To assess the robustness of our method to this design choice, we perform an ablation study by sweeping $\alpha$ and measuring the downstream clustering performance.

Figure~\ref{fig:alpha-ablation-lfr} shows the results on the LFR benchmark with $500$ nodes.
For each mixing parameter $\mu$, we run multiple random seeds and report the mean and standard deviation of ARI.
Overall, SRC(MST) is largely insensitive to $\alpha$ across the tested range of $\mu$, suggesting that its performance does not rely on a particular degree of locality in the lazy random walk measure.
In contrast, SRC(SPT) exhibits a mild dependence on $\alpha$ in the transition region (around $\mu\approx0.4$--$0.5$), where different $\alpha$ values lead to visibly different ARI means and increased variability.
These observations indicate that the interaction between the tree construction (SPT vs.\ MST) and the locality of $\mu_x^{\alpha}$ can affect performance near the detectability threshold, while both variants remain stable in the easy and hard regimes.

\begin{figure}[t]
  \centering
    \includegraphics[width=0.9\linewidth]{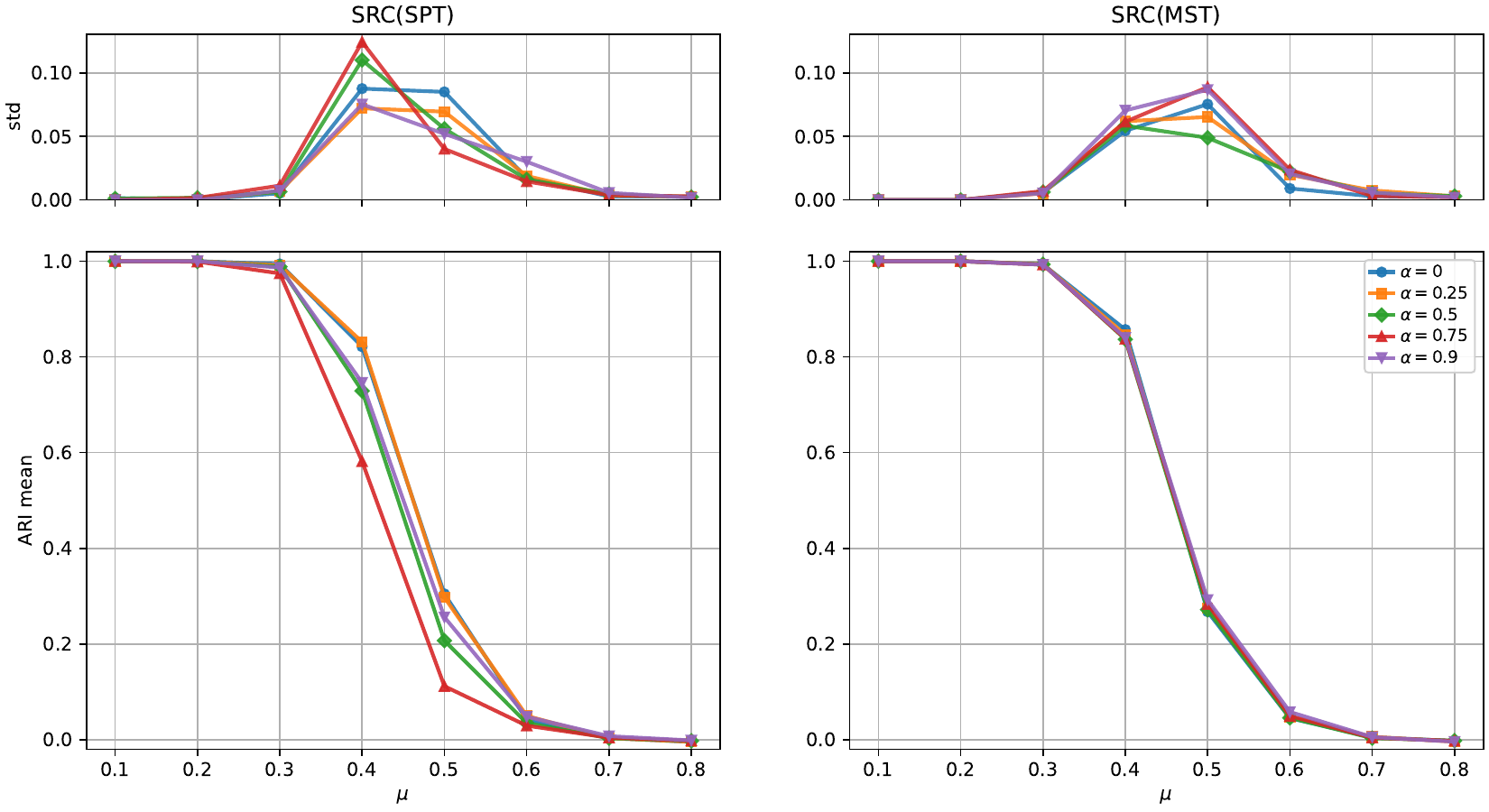}
  \caption{
    Ablation on the lazy random walk parameter $\alpha$ on LFR ($500$ nodes).
    We sweep $\alpha$ and plot the standard deviation (top row) and the mean (bottom row) of ARI over multiple random seeds as a function of the mixing parameter $\mu$.
    Left: SRC(SPT). Right: SRC(MST).
  }
  \label{fig:alpha-ablation-lfr}
\end{figure}

\subsection{Sensitivity to $\alpha$ on small networks (\texttt{Karate})}\label{app:sensitivity_to_alpha_small_nw}

As discussed in Section~\ref{sec:results_on_srf}, the \texttt{karate} network exhibits relatively poor ARI at $\alpha = 0.5$.
Figure~\ref{fig:alpha_ablation_real}(a) shows that this behavior is observed for both SRC(MST) and SRC(SPT).
While SRC(SPT) sometimes attains a higher mean ARI, the corresponding standard deviation is large, indicating unstable clustering results across runs.

We attribute this instability primarily to the extremely small graph size of the \texttt{karate} network.
In contrast, as shown in Fig.~\ref{fig:alpha_ablation_real}(b), the football network with 115 nodes exhibits consistently high ARI, and the performance is largely insensitive to the choice of $\alpha$ for both SRC(MST) and SRC(SPT).

These results suggest that the sensitivity to $\alpha$ is amplified in very small graphs, where stochastic effects and minor structural perturbations can significantly influence the detected communities.

\begin{figure}[t]
    \centering
    \begin{subfigure}[t]{0.48\linewidth}
        \centering
        \includegraphics[width=\linewidth]{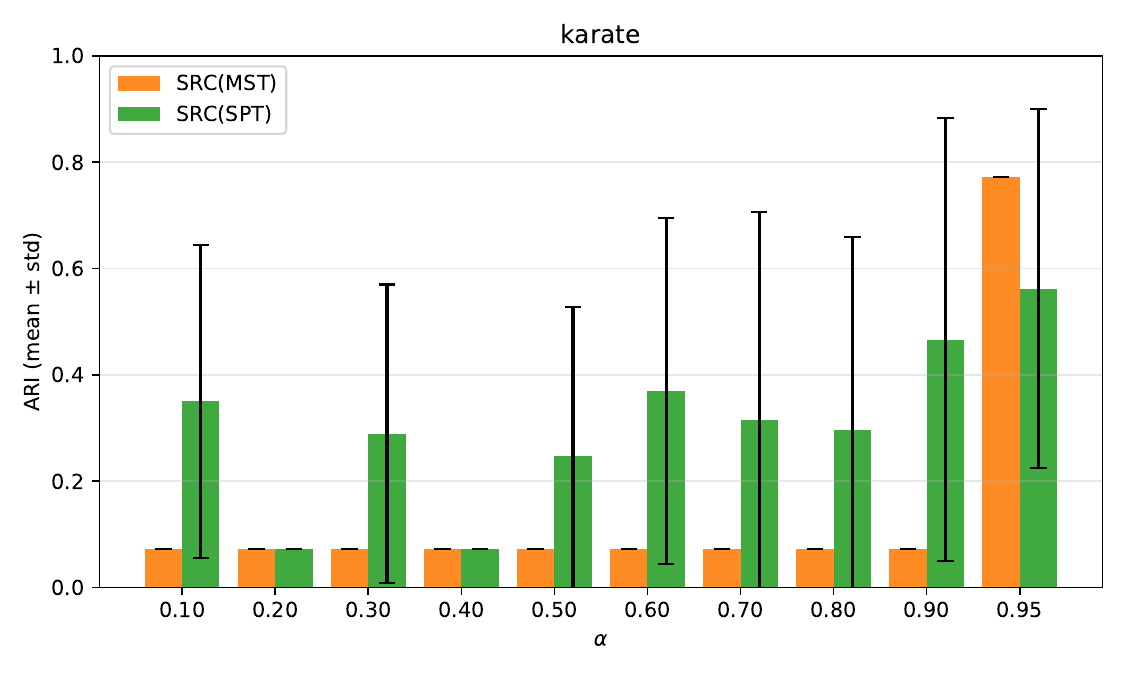}
        \caption{Karate}
        \label{fig:alpha_ablation_karate}
    \end{subfigure}\hfill
    \begin{subfigure}[t]{0.48\linewidth}
        \centering
        \includegraphics[width=\linewidth]{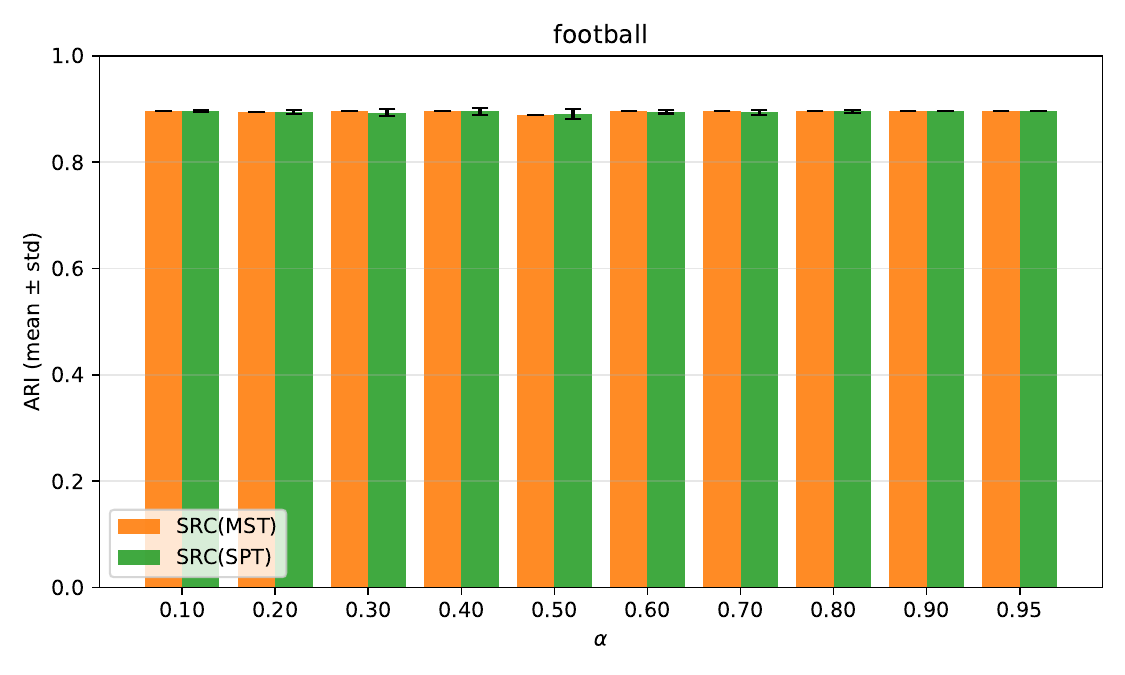}
        \caption{Football}
        \label{fig:alpha_ablation_football}
    \end{subfigure}
    \caption{Sensitivity to $\alpha$ on real-world networks. Error bars indicate std over runs.}
    \label{fig:alpha_ablation_real}
\end{figure}

\subsection{Ablation Study: Sensitivity to the Length-to-Similarity Temperature $\beta$}\label{sec:ablation_beta}

Our Ricci-flow procedures (SRC/ORC) evolve a set of \emph{edge weights} $\{w_{xy}\}_{\langle x,y\rangle\in E}$,
which are used as a \emph{length-like quantity} in shortest-path computations.
In contrast, modularity-based community detection algorithms such as Louvain operate on \emph{similarity weights}.
To apply Louvain after the flow, we therefore convert the learned values $w_{xy}$ into similarity weights
via a monotone transformation as in~\eqref{eq:length_to_similarity}
\begin{equation}
\label{eq:length_to_similarity_beta}
    \tilde{w}_{xy} \;=\; \exp(-\beta\, w_{xy}),
\end{equation}
where $\beta>0$ controls the sharpness of the conversion.
Larger $\beta$ emphasizes small values of $w_{xy}$ more strongly, while smaller $\beta$ yields a smoother weighting.
Since $\beta$ plays the role of a temperature parameter in an RBF-like kernel, it is natural to ask how sensitive the downstream clustering performance is to this choice.

\paragraph{Setup.}
We conduct this ablation on the LFR benchmark, using SRC(SPT) for efficiency.
For each mixing parameter $\mu$, we run the SRF for a fixed number of iterations and then apply Louvain clustering to the transformed similarities $\tilde{w}_{xy}$ in \eqref{eq:length_to_similarity_beta}.
To avoid favoring any particular $\beta$, we perform the same resolution grid search for Louvain under every $\beta$ value.
We report the mean and standard deviation over multiple random LFR instances.

\paragraph{Results.}
Figure~\ref{fig:result_lfr_beta_ari_and_res_2panel} summarizes the effect of $\beta$ on (left) ARI and (right) the selected Louvain resolution.
For moderate values of $\beta$ (e.g., $\beta \in \{0.01, 0.1, 1\}$), SRC(SPT)+Louvain is highly stable: the ARI curves almost overlap in the recoverable regimes (e.g., $\mu \le 0.3$), and the transition occurs around similar $\mu$ with comparable variability.
In contrast, larger $\beta$ values lead to noticeably earlier degradation. In particular, $\beta=5$ and $\beta=10$ reduce ARI substantially starting around the intermediate regime (e.g., $\mu \approx 0.4$--$0.5$), indicating that overly sharp length-to-weight conversion can hinder downstream community recovery.

The selected Louvain resolution (right panel) varies smoothly with $\mu$ and shows only moderate sensitivity to $\beta$ for $\beta \le 5$, while $\beta=10$ tends to select slightly smaller resolutions especially for larger $\mu$ (e.g., $\mu \gtrsim 0.5$).

\paragraph{Default choice.}
Overall, the ablation supports that SRF-based community detection is \emph{insensitive} to $\beta$ within a reasonable range,
but can fail when $\beta$ is excessively large.
Unless otherwise stated, we fix $\beta=1$ throughout the paper.
This choice is natural because the flow applies normalization at each iteration, yielding a stable scale for $w_{xy}$,
and thus $\beta=1$ serves as a robust default without additional tuning.

\begin{figure}[t]
    \centering
    \includegraphics[width=0.90\linewidth]{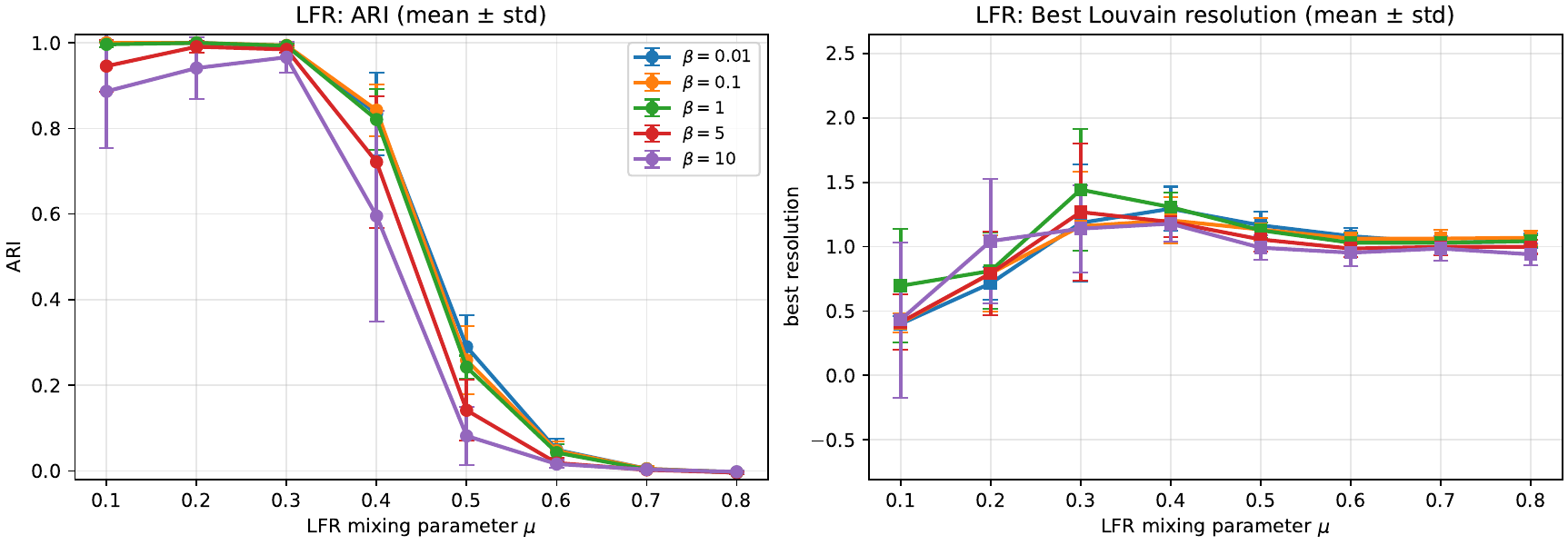}
    \caption{LFR benchmark: sensitivity to the length-to-weight temperature $\beta$ in $\tilde{w}_{xy}=\exp(-\beta w_{xy})$.
Left: ARI (mean $\pm$ std) of SRC(SPT)+Louvain after the SRF.
Right: best Louvain resolution selected by grid search (mean $\pm$ std).
Performance is stable for moderate $\beta$ (0.01--2), while overly large $\beta$ (5,10) collapses weights and degrades modularity-based clustering.}
    \label{fig:result_lfr_beta_ari_and_res_2panel}
\end{figure}

\subsection{Curvature histograms on LFR graphs}
\label{app:curvature_histograms_lfr}

To better understand the empirical behavior of SRC and how it differs from ORC, 
we visualize the distributions of edge-wise curvature values on LFR benchmark graphs.
Figure~\ref{fig:ricci_curvature_mu_compare} shows representative histograms for two mixing parameters, $\mu = 0.10$ and $\mu = 0.40$, 
where smaller $\mu$ corresponds to more clearly separated communities observed in Fig.~\ref{fig:sbm_lfr_real_runtime} (b).

\paragraph{ORC exhibits bimodality for highly separable graphs.}
A notable observation is that for $\mu = 0.10$, the ORC histogram displays a clear bimodal structure (two separated peaks).
This suggests that ORC assigns systematically different curvature values to two classes of edges, 
which can be interpreted as edges inside communities versus edges bridging across communities.
In other words, in an easy regime where the community structure is strong, ORC provides a sharp separation of edge types,
leading to high discriminability in curvature space.

\paragraph{SRC produces higher peaks but heavier negative tails.}
In contrast, SRC (both SRC(SPT) and SRC(MST)) tends to produce a higher concentration of curvature values, 
yielding sharper peaks than ORC.
At the same time, SRC also exhibits a noticeably longer tail toward low curvature values.
This behavior indicates that while SRC assigns relatively large curvature values to many edges (making the main mass more concentrated),
it still produces a non-negligible fraction of edges with strongly negative curvature.
Such negatively curved edges are likely to correspond to structurally important bridges or shortcuts, which SRC emphasizes more strongly.

\paragraph{Interpretation: tree-based aggregation induces smoothing yet preserves extreme edges.}
This systematic difference can be intuitively explained by the tree-based formulation of SRC.
Since SRC replaces the Wasserstein cost in ORC by ST on a tree, 
local mass discrepancies are aggregated along the tree paths.
This aggregation introduces a smoothing effect, pushing the bulk of curvature values upward compared to ORC.
Nevertheless, edges that align with long-range discrepancies in the tree-metric can accumulate large transport costs,
resulting in a long negative tail.
This observation is consistent with the qualitative trend described in Appendix~\ref{sec:intuition_on_curvature_values},
where SRC is often larger than ORC on the same graph, yet exhibits distinct extreme behavior.

\paragraph{Effect of increasing $\mu$.}
When $\mu$ increases to $0.40$, community separation becomes weaker and the curvature histograms become less structured.
The bimodality observed in ORC at $\mu=0.10$ diminishes, 
and the distributions of SRC and ORC become closer in shape, reflecting the reduced distinction between intra- and inter-community edges.
Overall, these histograms highlight that SRC retains similar qualitative sensitivity to community structure as ORC,
while producing systematically shifted and more concentrated curvature values.

\paragraph{Why can SRC still improve community detection even without clear histogram separation?}
Although the SRC histograms are not as clearly bimodal as ORC at $\mu=0.10$, 
SRC still yields strong Louvain-based community separation (see Fig.~\ref{fig:sbm_lfr_real_runtime}).
This is because the histogram visualizes only the \emph{marginal distribution} of curvature values,
while Louvain clustering depends on the \emph{graph-structured assignment} of edge weights.

Ablation results in Appendix~\ref{app:ablation_study_on_comm_detection} further show that directly clustering by thresholding
the learned lengths (length-cut) is considerably less reliable, exhibiting substantially larger variance and degraded ARI,
whereas Louvain-based post-processing remains stable.
Therefore, rather than requiring explicit separation in the curvature histogram, effective clustering in our pipeline is better
characterized empirically by the downstream partitions produced by Louvain from the flow-updated weights.
We leave a finer-grained analysis of how curvature-induced reweighting interacts with modularity optimization to future work.

\begin{figure}[t]
    \centering
    \begin{subfigure}[b]{0.48\linewidth}
        \centering
        \includegraphics[width=\linewidth]{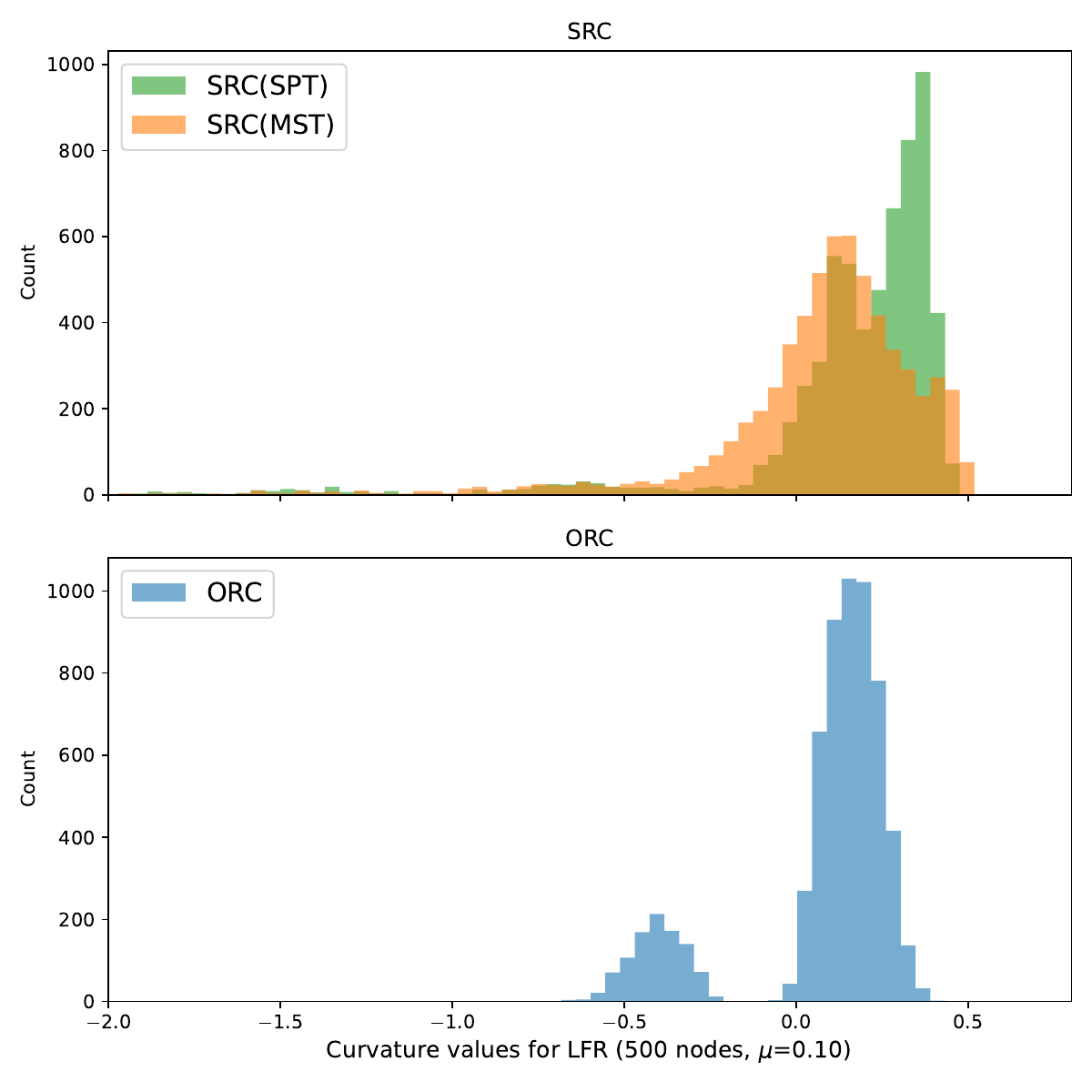}
        \caption{$\mu = 0.10$}
        \label{fig:ricci_mu010}
    \end{subfigure}
    \hfill
    \begin{subfigure}[b]{0.48\linewidth}
        \centering
        \includegraphics[width=\linewidth]{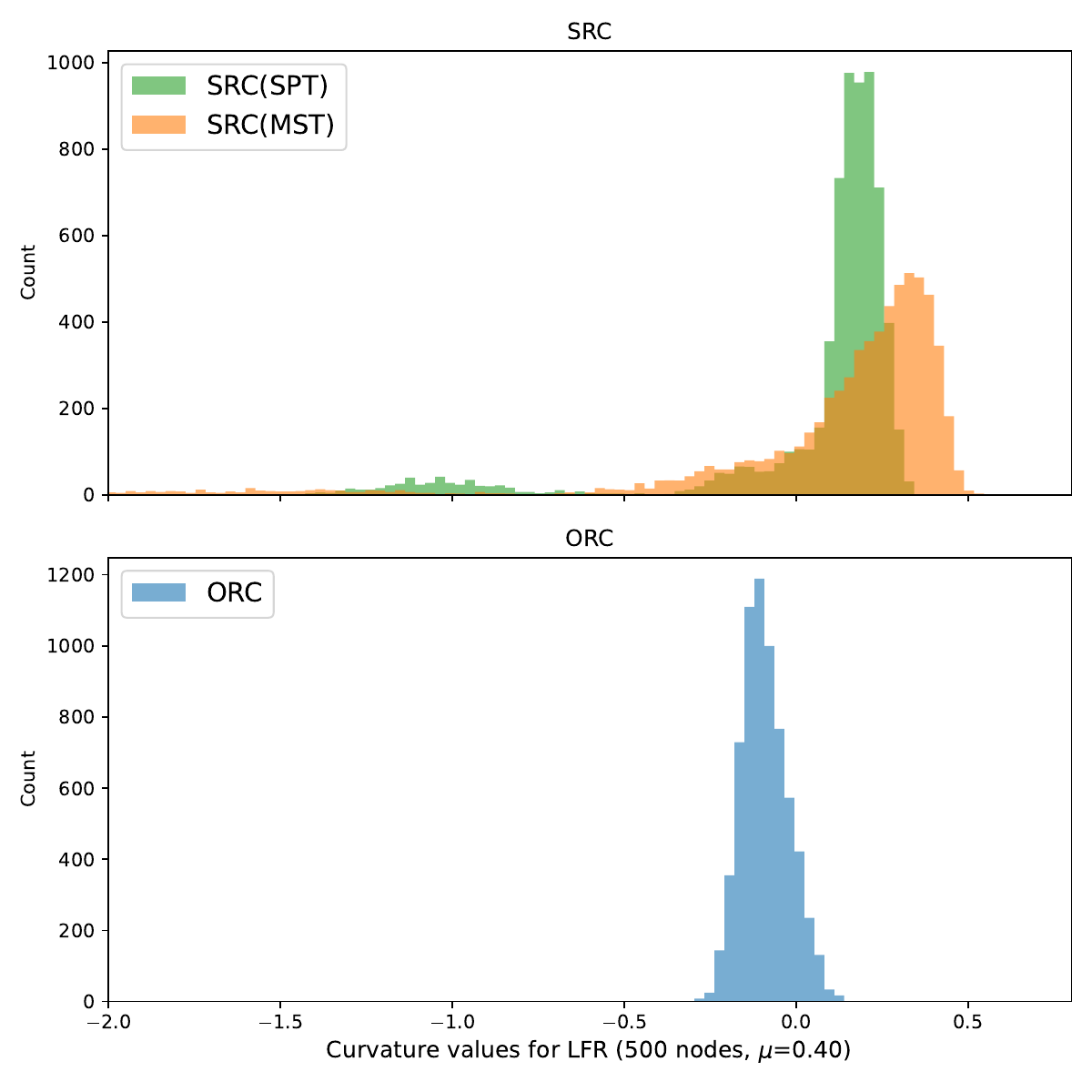}
        \caption{$\mu = 0.40$}
        \label{fig:ricci_mu040}
    \end{subfigure}

    \caption{Ricci curvature distributions for different mixing parameters $\mu$.}
    \label{fig:ricci_curvature_mu_compare}
\end{figure}

\subsection{Robustness to Tree Construction (SPT vs. MST vs. Random Spanning Tree)}\label{app:tree_robustness}

\begin{figure*}[t]
    \centering
    \begin{subfigure}[t]{0.32\textwidth}
        \centering
        \includegraphics[width=\linewidth]{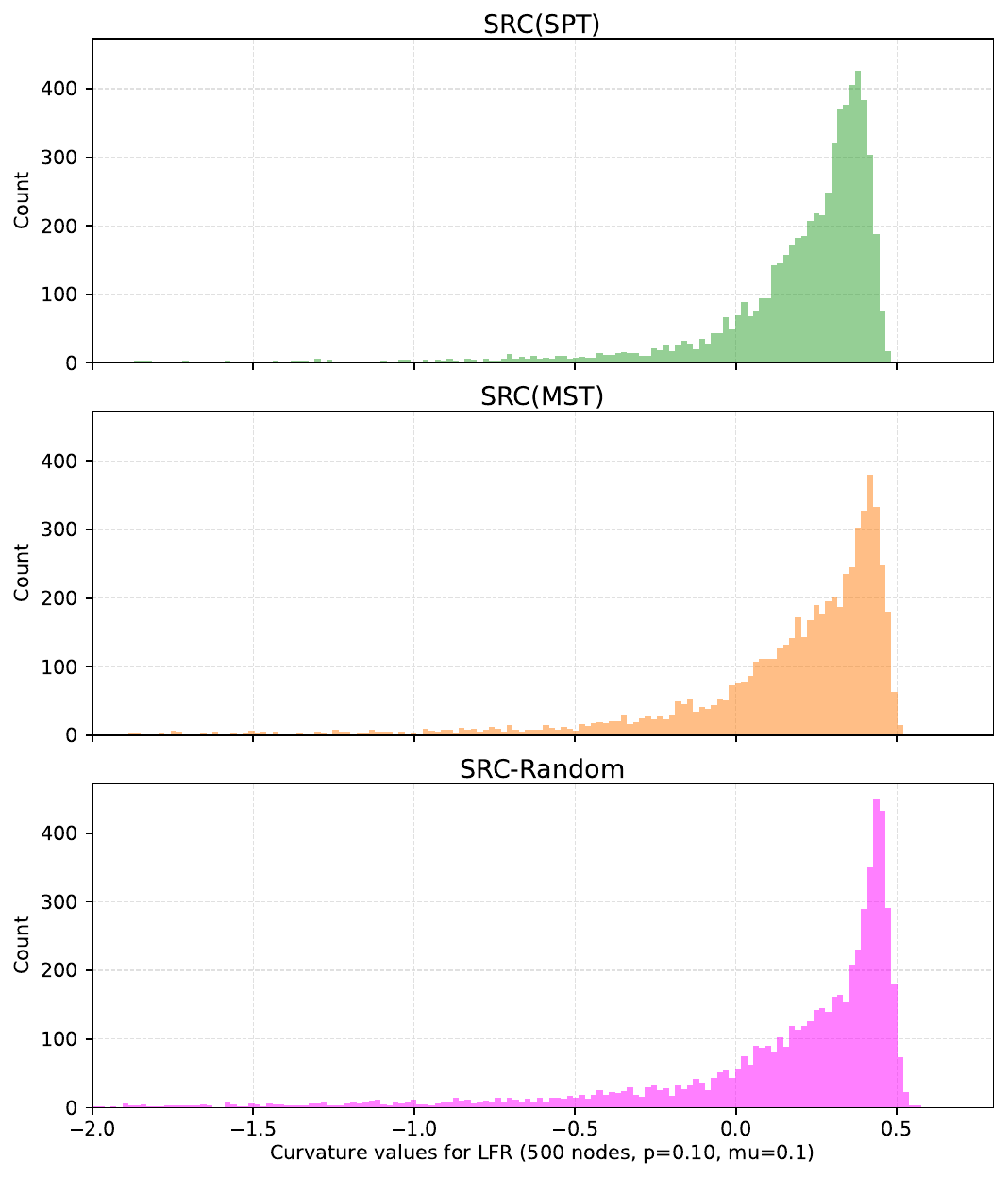}
        \caption{$\mu=0.1$}
        \label{fig:src_tree_robust_mu01}
    \end{subfigure}
    \hfill
    \begin{subfigure}[t]{0.32\textwidth}
        \centering
        \includegraphics[width=\linewidth]{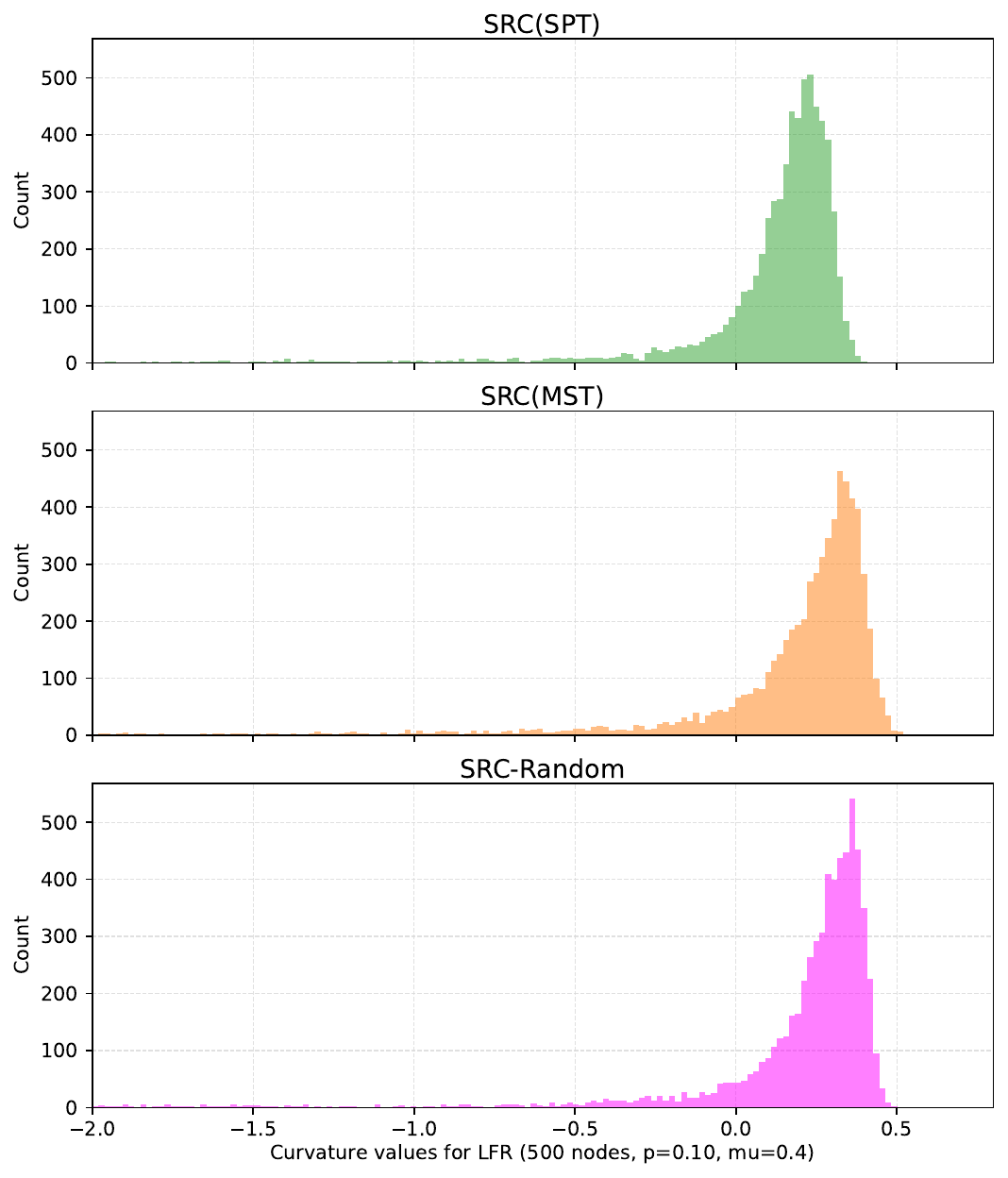}
        \caption{$\mu=0.4$}
        \label{fig:src_tree_robust_mu04}
    \end{subfigure}
    \hfill
    \begin{subfigure}[t]{0.32\textwidth}
        \centering
        \includegraphics[width=\linewidth]{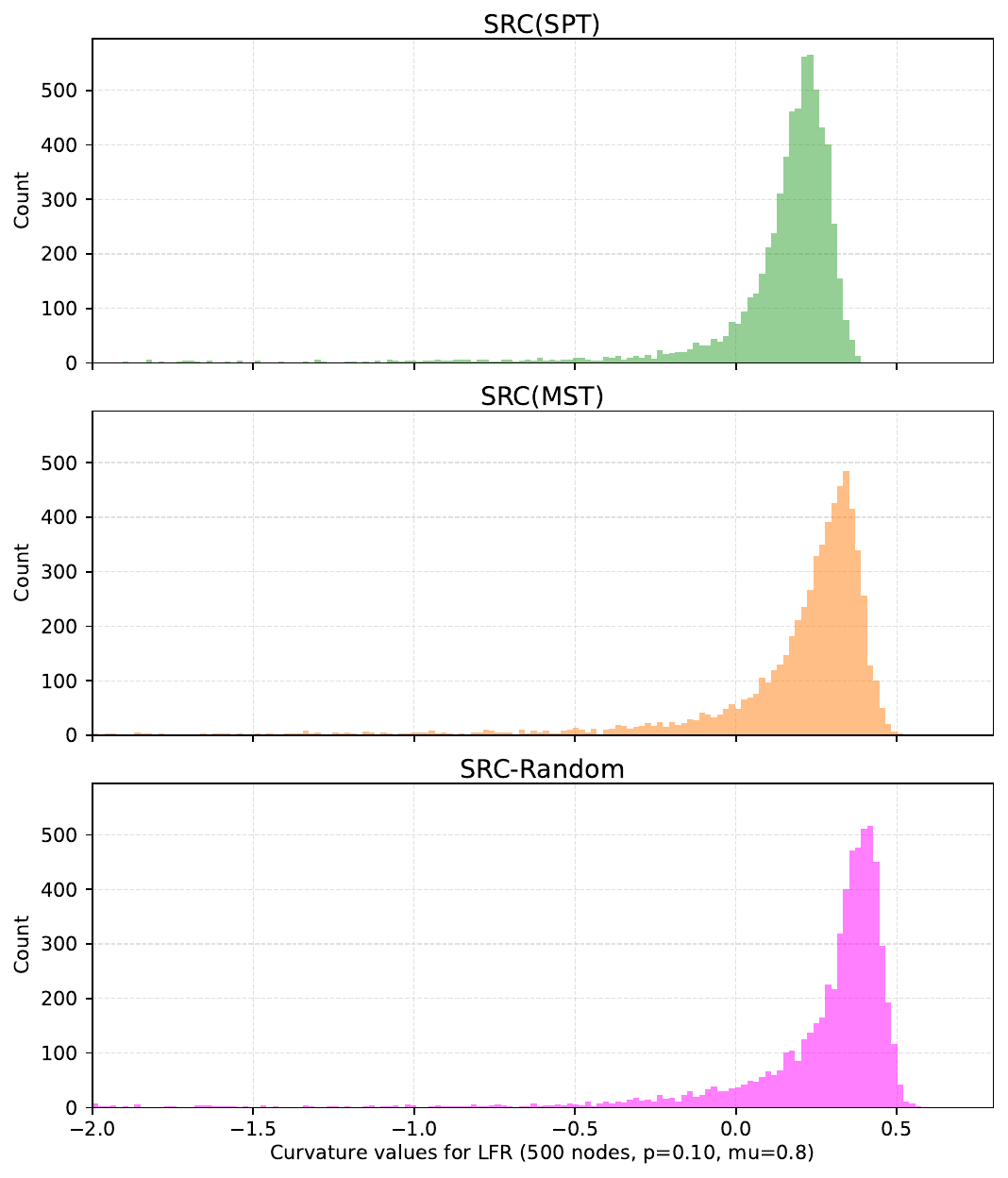}
        \caption{$\mu=0.8$}
        \label{fig:src_tree_robust_mu08}
    \end{subfigure}

    \caption{\textbf{Robustness to tree construction.} SRC distributions on the same LFR graphs computed using different extracted trees (SPT, MST, Random spanning tree). Across $\mu$, the global distributional shapes are similar, suggesting SRC mainly reflects coarse graph geometry rather than a specific tree optimality criterion.}
    \label{fig:src_tree_robustness}
\end{figure*}

SRC is computed on a tree extracted from the original graph, which raises a natural question:
to what extent does SRC depend on \emph{how} the tree is constructed?
In the main paper we primarily use structured choices such as SPT or MST,
as they are common and often considered to better preserve certain aspects of the graph geometry.
Here, we empirically examine whether the \emph{optimality criterion} of the extracted tree is essential,
or whether using an arbitrary spanning tree already yields a similar curvature profile.

\paragraph{Setup.}
We fix the underlying graph and change only the extracted tree used by SRC.
Specifically, for each LFR instance we compute SRC on (i) an SPT, (ii) an MST, and (iii) a random spanning tree sampled from the same graph,
while keeping all other settings unchanged.
Figure~\ref{fig:src_tree_robustness} reports the resulting SRC histograms for three mixing parameters $\mu\in\{0.1,0.4,0.8\}$.

\paragraph{Observation.}
Across $\mu$, the three tree constructions produce curvature distributions with a similar global shape:
the main mass of the distribution and the overall skew/tail behavior are largely consistent across SPT, MST, and random spanning trees.
We do observe non-negligible differences at a finer level (e.g., local shifts of the peak or tail thickness),
indicating that SRC is not strictly invariant to the tree choice; however, the coarse distributional profile remains stable.

\section{Edge pruning experiments with curvatures}\label{sec:orc_manl_additional_content}

In this section we present additional analyses on curvature-based edge pruning in
Section~\ref{sec:edge_pruning_with_Curvature}.
Before presenting the results, we summarize the benchmark datasets and data-generation
protocol used in these experiments, including the synthetic datasets adopted from ORC-MANL
(See Section~\ref{sec:profiles_of_datasets_for_manl}).

We begin by visualizing the synthetic 2D and 3D datasets used in the experiments,
highlighting the presence of shortcut edges that deviate from the underlying
manifold structure (See Section~\ref{sec:visualization_of_datasets_for_edge_pruning}). 
We then provide extended pruning results across ten benchmark datasets,
evaluating the reliability of SRC- and ORC-based pruning schemes in comparison
to distance-only baselines (See Section~\ref{sec:edge_pruning_for_all}). 
To aid interpretation, we also examine the distribution of curvature values,
offering intuition for why SRC yields more effective filtering of shortcuts (See Section~\ref{sec:intuition_on_curvature_values}).
Finally, we remark on the overall pruning procedure to clarify the roles of the
main parameters (See Section~\ref{sec:remark_on_procedure}).
Together, these results reinforce the main-text findings that SRC-based pruning,
especially when combined with MANL, consistently outperforms ORC and
distance-only approaches.

\subsection{Profiles of datasets}\label{sec:profiles_of_datasets_for_manl}

We additionally include the synthetic benchmarks used in ORC-MANL
to evaluate curvature-based edge pruning under the same protocol.

\begin{itemize}
    \item \textbf{ORC-MANL datasets.}
        We additionally include the synthetic benchmarks used in ORC-MANL.
        Specifically, we generate 10 synthetic datasets using the reference implementation provided at
        \href{https://github.com/TristanSaidi/orcml}{\texttt{github.com/TristanSaidi/orcml}} (MIT license).
        We used the authors’ reference implementation and applied minor compatibility fixes to reproduce the dataset generation procedure.
\end{itemize}

\subsection{Visualization of datasets for edge pruning}\label{sec:visualization_of_datasets_for_edge_pruning}

We provide an overview of the datasets employed for edge pruning through visualization.
Fig.~\ref{fig:five_2D_datasets} depicts a two-dimensional configuration, whereas Fig.~\ref{fig:five_3D_datasets_with_shortcut_edges} presents a three-dimensional distribution, each comprising $4,000$ data points. Edges that deviate from the underlying manifold are referred to as shortcut edges, which are highlighted in red in the figures.

The subsequent analysis in Section~\ref{sec:edge_pruning_with_Curvature} examines whether SRC and ORC are capable of excluding these shortcut edges in a sufficient and reliable manner.

\begin{figure}[ht]
    \centering
    \includegraphics[width=1\linewidth]{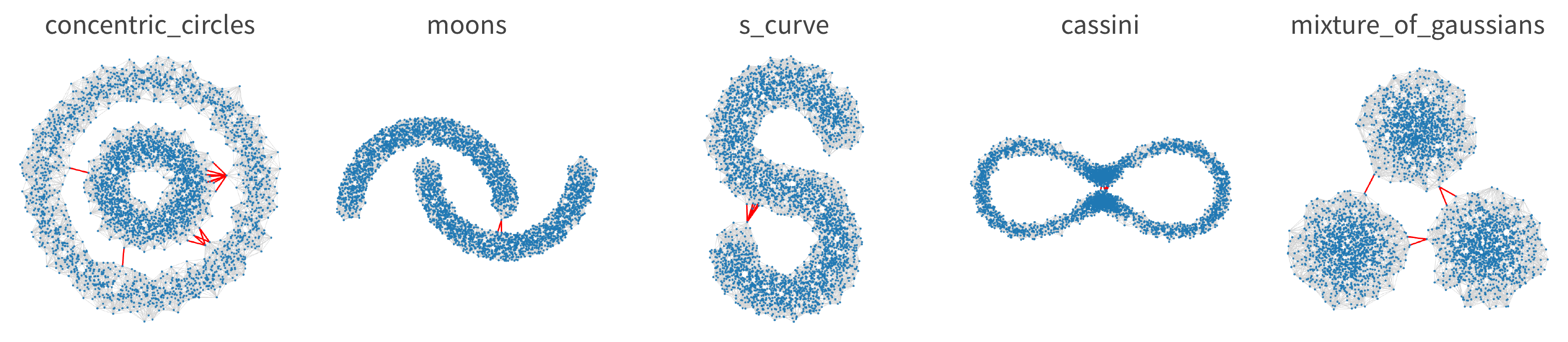}
    \caption{Two-dimensional synthetic datasets employed in the edge pruning experiments. Each dataset consists of 4,000 data points. Red lines indicate shortcut edges, i.e., edges that do not follow the underlying manifold structure.}
    \label{fig:five_2D_datasets}
\end{figure}

\begin{figure}[ht]
    \centering
    \includegraphics[width=0.99\linewidth]{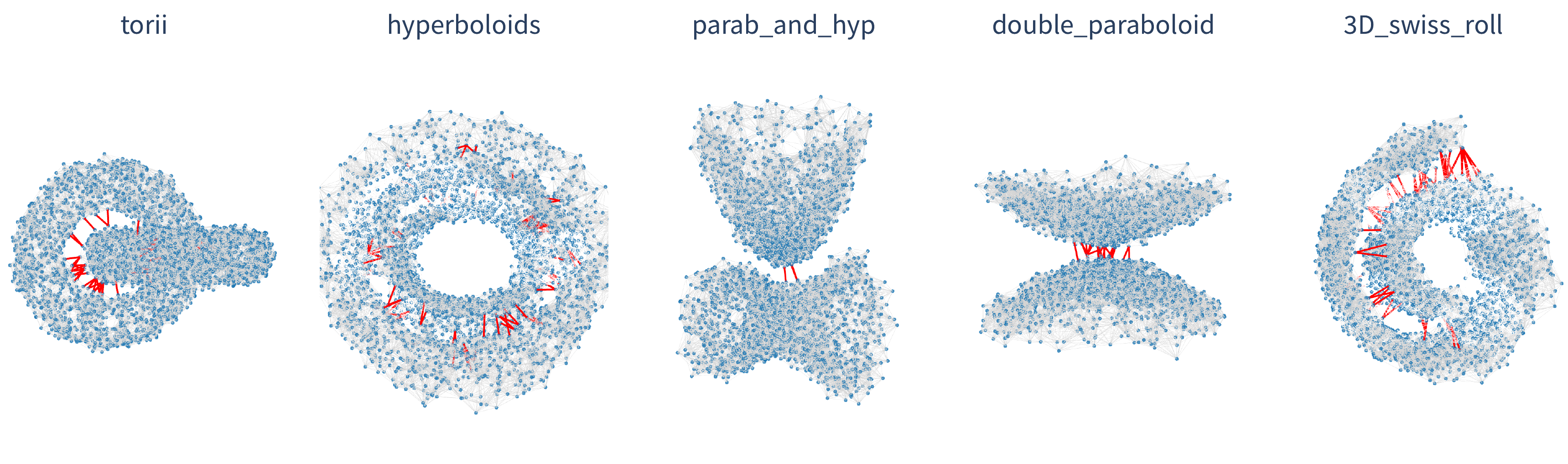}
    \caption{Three-dimensional synthetic datasets employed in the edge pruning experiments. Each dataset contains 4,000 data points. Red lines indicate shortcut edges, i.e., edges that deviate from the underlying manifold structure.}
    \label{fig:five_3D_datasets_with_shortcut_edges}
\end{figure}

\subsection{Extended Results for Edge Pruning}\label{sec:edge_pruning_for_all}

We report the complete edge pruning results across all ten benchmark datasets introduced in Section~\ref{sec:edge_pruning_with_Curvature} used in~\citep{DBLP:conf/iclr/SaidiHB25}: 
\texttt{concentric\_circles}, \texttt{moons}, \texttt{s\_curve}, \texttt{cassini}, \texttt{mixture\_of\_gaussians}, 
\texttt{torii}, \texttt{hyperboloids}, \texttt{parab\_and\_hyp}, \texttt{double\_paraboloid}, and \texttt{3D\_swiss\_roll}. 
Each dataset consists of $4,000$ samples.

The experiments are conducted under the parameter pair $\delta_{\text{M}}=0.75$ and $\lambda_{\text{M}}=0.01$, which achieve the best scores for both ORC-MANL and SRC-MANL 
within the parameter ranges 
$\delta_{\text{M}} \in [0.7,0.9]$ (step 0.05) and 
$\lambda_{\text{M}} \in \{10^{-3},10^{-2},0.1,0.2,0.5\}$.
Figures~\ref{fig:results_edge_pruning} summarize the percentage of shortcut edges correctly removed (red) and good edges incorrectly removed (light blue). 

\vspace{0.5em}
\noindent\textbf{Observations.}  
Across all datasets, the same overall trend holds: 
\begin{itemize}
    \item ``SRC only'' pruning already removes shortcut edges reliably while keeping most good edges intact.  
    \item Adding MANL on top of SRC further improves precision, leading to the best overall pruning quality.  
    \item By contrast, Distance-only pruning fails to separate shortcuts in most datasets, although in a few cases 
    (e.g., \texttt{concentric\_circles}, \texttt{moons}, \texttt{s\_curve} and \texttt{mixture\_of\_gaussians}) they accidentally capture some shortcuts.  
    \item ``ORC only'' pruning remains unstable, often misclassifying a large number of good edges; even with MANL, it only manages to reach performance comparable to the SRC-based counterpart. 
\end{itemize}

\begin{figure}[htbp]
    \centering
    \includegraphics[width=0.98\linewidth]{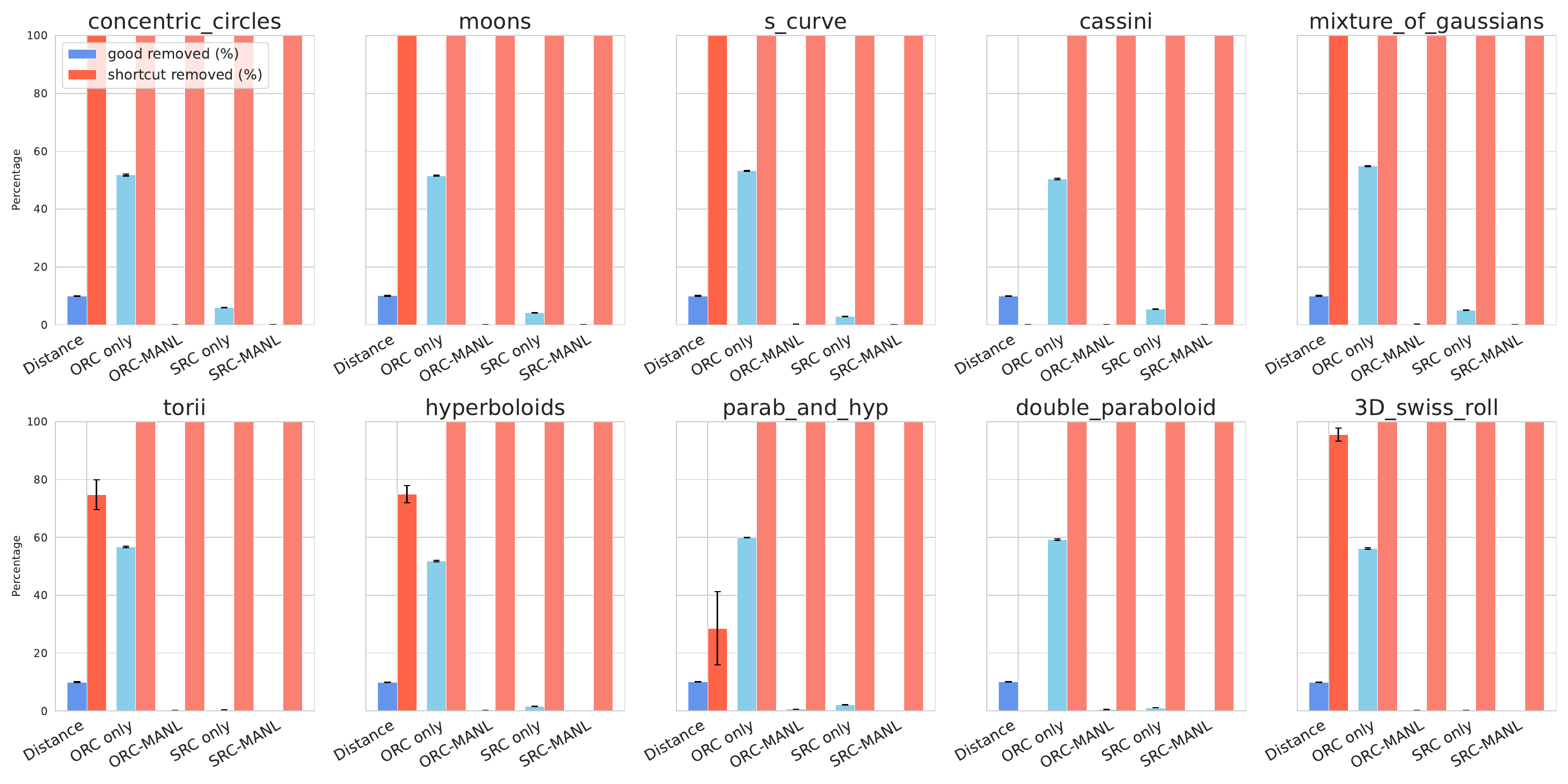}
    \caption{
Bars indicate the percentage of shortcut edges correctly removed (red) 
and good edges incorrectly removed (light blue). 
While Distance-only and ORC pruning often fail, ``SRC only'' achieves reliable shortcut detection with fewer mistakes. 
With MANL, ORC and SRC(SPT) consistently succeed.}
    \label{fig:results_edge_pruning}
\end{figure}

\vspace{0.5em}
\noindent Overall, these extended results confirm the trends shown in the main text: Distance-only pruning fails to separate shortcuts, ``ORC only'' suffers from excessive misclassification, and SRC(SPT) consistently yields the most reliable pruning both alone and within the MANL framework.

\subsection{Intuition on curvature values}\label{sec:intuition_on_curvature_values}

Empirically, we observe that SRC(SPT) values are often larger than those of ORC 
on the same graph as in Fig.~\ref{fig:two_circles_curvature}. This implies that, under an identical thresholding scheme, 
the first-stage filtering with SRC preferentially selects edges that are 
geometrically close to shortcuts. As a result, the subsequent shortcut-detection 
step becomes more effective, explaining why SRC-MANL achieves higher pruning 
accuracy compared to the ORC-based counterpart.

\begin{wrapfigure}{r}{0.48\linewidth}
\vspace{-3pt}
\centering
\includegraphics[width=\linewidth]{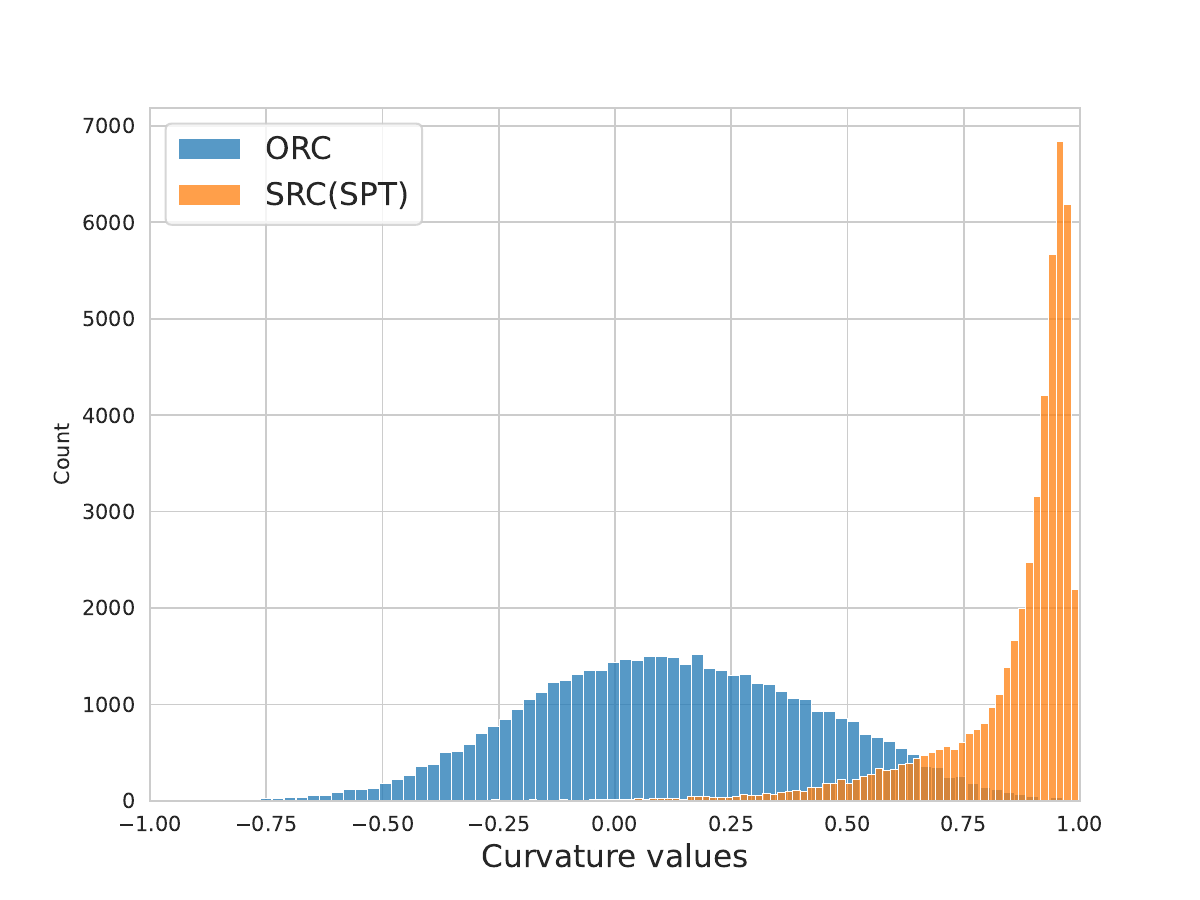}
\caption{Illustrative example on the \texttt{3D\_swiss\_roll} dataset, where we see in Fig.~\ref{fig:five_2D_datasets}. Distribution of curvature values, showing that SRC(SPT) tends to assign larger curvature values than ORC. This leads to more effective identification of shortcut edges in the first-stage filtering.}
\label{fig:two_circles_curvature}
\vspace{-5pt}
\end{wrapfigure}

This behavior can be intuitively understood: 
Since SRC relies on tree-based ST, local fluctuations in neighborhood measures are aggregated along paths of the tree. 
This aggregation has a smoothing effect, leading to systematically larger curvature values compared to ORC, which directly computes Wasserstein distances on raw neighborhoods.

\subsection{Remark on the procedure}\label{sec:remark_on_procedure}
The ORC-MANL procedure consists of two stages: 
(i) a curvature-based filtering step, where edges with curvature below a threshold 
controlled by $\delta_{\text{M}} \in (0,1)$ are selected as candidates; and 
(ii) a shortcut-detection step is compared 
against a threshold determined by the regularization parameter $\lambda_{\text{M}} > 0$. 
These parameter roles are consistent with the original ORC-MANL implementation 
\citep{DBLP:conf/iclr/SaidiHB25}.

\end{document}